\documentclass{bmvc2k}

\usepackage{amsmath}
\usepackage{amssymb}
\usepackage{booktabs}
\usepackage{makecell}
\usepackage{pifont}
\usepackage{color}
\newcommand{\PreserveBackslash}[1]{\let\temp=\\#1\let\\=\temp}
\newcolumntype{C}[1]{>{\PreserveBackslash\centering}p{#1}}
\newcolumntype{R}[1]{>{\PreserveBackslash\raggedleft}p{#1}}
\newcolumntype{L}[1]{>{\PreserveBackslash\raggedright}p{#1}}

\usepackage[normalem]{ulem}
\usepackage{enumitem}
\usepackage[rightcaption]{sidecap}
\usepackage{multirow}
\usepackage{marvosym}

\newtheorem{proposition}{Proposition}
\newtheorem{proof}{Proof}

\title{ESAD: End-to-end Semi-supervised Anomaly Detection}

\addauthor{Chaoqin Huang\textsuperscript{1}}{huangchaoqin@sjtu.edu.cn}{2}
\addauthor{Fei Ye}{yf3310@sjtu.edu.cn}{1}
\addauthor{Peisen Zhao\textsuperscript{1}}{pszhao@sjtu.edu.cn}{3}
\addauthor{Ya Zhang \textsuperscript{\Letter 1}}{ya_zhang@sjtu.edu.cn}{2}
\addauthor{Yanfeng Wang\textsuperscript{1}}{wangyanfeng@sjtu.edu.cn}{2}
\addauthor{Qi Tian}{tian.qi1@huawei.com}{3}

\addinstitution{
 Cooperative Medianet Innovation Center,\\
 Shanghai Jiao Tong University
}
\addinstitution{
 Shanghai AI Laboratory
}
\addinstitution{
Huawei Cloud \& AI
}

\runninghead{Huang et al.}{ESAD: End-to-end Semi-supervised Anomaly Detection}

\begin{document}

\maketitle

\begin{abstract}
This paper explores semi-supervised anomaly detection, a more practical setting for anomaly detection where a small additional set of labeled samples are provided. We propose a new KL-divergence based objective function for semi-supervised anomaly detection, and show that two factors: the \emph{mutual information} between the data and latent representations, and the \emph{entropy} of latent representations, constitute an integral objective function for anomaly detection. To resolve the contradiction in simultaneously optimizing the two factors, we propose a novel encoder-decoder-encoder structure, with the first encoder focusing on optimizing the mutual information and the second encoder focusing on optimizing the entropy. The two encoders are enforced to share similar encoding with a consistent constraint on their latent representations. Extensive experiments have revealed that the proposed method significantly outperforms several state-of-the-arts on multiple benchmark datasets, including medical diagnosis and several classic anomaly detection benchmarks. 
\end{abstract}

%-------------------------------------------------------------------------
\section{Introduction}
\label{sec:intro}
Anomaly detection (AD), with broad application in medical diagnosis \cite{tuluptceva2020anomaly}, credit card fraud detection \cite{phua2010comprehensive}, and autonomous driving \cite{eykholt2018robust}, has received significant attention among the machine learning community. The main challenge in AD is that, it is prohibitive, even if not impossible, to collect a representative set of anomalous samples due to its remarkable scarcity in the population. To bypass the challenge, many approaches \cite{Sabokrou2018Adversarially,perera2019ocgan,zhang2020p} have resorted to unsupervised learning so that only normal samples are needed for model training.

\begin{figure}[t]
\centering
\includegraphics[width=1.0\textwidth]{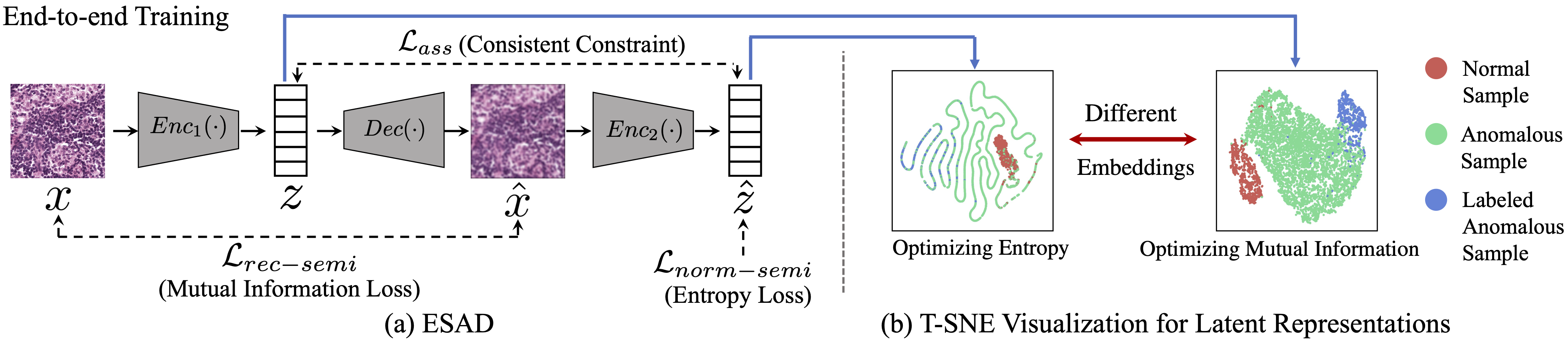}
\vspace{-18pt}
\caption{The training processes of ESAD for semi-supervised anomaly detection. (a) ESAD leverages an encoder-decoder-encoder structure, where the two encoders are enforced to share similar encoding with a consistent constraint on their latent representations, with the first encoder targeting to optimize the mutual information and the second encoder focusing on the entropy. (b) shows the T-SNE \cite{maaten2008visualizing} visualization results for the latent representations.}
\label{img:esad}
\vspace{-15pt}
\end{figure}

Semi-supervised anomaly detection, where a small set of labeled data are provided for training in addition to a large amount of unlabeled data, represents a more practical setting of anomaly detection. In the real-world scenario, it is feasible to obtain a small set of `biased' anomalous data. Earlier semi-supervised AD methods follow the unsupervised learning paradigm and employ such a labeled anomalous set through a certain form of regularization \cite{munoz2010semisupervised,gornitz2013toward}. More recently, Deep SAD~\cite{SAD}, the first deep semi-supervised AD method, builds upon the \emph{Infomax principle}~\cite{linsker1988self,bell1995information,hjelm2019learning} that maximizes the mutual information between the data and the latent representations and enforces an additional regularization on the latent representations. Due to the contradiction between the mutual information-based objective and entropy-based regularization, named model collapse in~\cite{ruff2018deep,SAD}, Deep SAD adopts a two-stage process: (i) autoencoder pre-training for mutual information maximization; and (ii) encoder fine-tuning for entropy regularization. This sequential learning process cannot guarantee the two objectives are simultaneously optimized and cannot well resolve the contradiction between the mutual information and entropy during the optimization. The model tends to collapse when the entropy is minimized to zero at the second stage, and the model inevitably leads to low mutual information as all data are mapped into a constant~\cite{ruff2018deep,SAD}.

In this paper, we introduce ESAD, an end-to-end method for semi-supervised anomaly detection. We start with exploring an alternative optimization target for AD by maximizing the KL-divergence between the normal and the anomalous class. Considering the challenge in estimating the anomalous distribution which results in the infeasibility of direct optimization, the KL-divergence based objective function is relaxed and further decomposed into two factors: (i) \emph{mutual information} between the data and the latent representations and (ii) \emph{entropy} of latent representations. While the two factors in the final objective function seem to be the same as those of Deep SAD, the difference lies in that, here mutual information and entropy are considered an integral part of the single objective function and need to be optimized simultaneously in an end-to-end training fashion.

In addition, to resolve the contradiction between the mutual information and entropy during optimization, we extend the autoencoder structure widely adopted for deep anomaly detection into an encoder-decoder-encoder structure illustrated in Figure~\ref{img:esad}~(a), where two separate but closely resembled encoders are employed to emphasize different factors in the optimization so that the model can be trained end-to-end.  
Specifically, although the two encoders are enforced to share similar encoding via a consistent constraint on the outputs, the first encoder focuses on mutual information through targeting on good representations only for the normal data but not for the labeled anomalous data, while the second encoder focuses on entropy by enforcing the compacted representations for the normal data and scattered representations for the anomalous data. With the encoder-decoder-encoder structure, we achieve end-to-end training for semi-supervised anomaly detection. Figure~\ref{img:esad}~(b) shows that the two encoders actually result in quite different embeddings, confirming the difficulty in finding a common embedding that simultaneously optimizing both mutual information and entropy. However, embeddings from both encoders show a better separation between the normal and anomalous classes than that of Deep SAD.

To validate the effectiveness of ESAD, we experiment with two medical image datasets~\cite{bejnordi2017diagnostic,wang2017chestx}, three natural image benchmarks~\cite{lecun1998mnist,xiao2017fashion,krizhevsky2009learning}, and several classic AD benchmarks \cite{Rayana2016}. Extensive results and analysis have shown that ESAD outperforms state-of-the-art methods on almost all datasets. Ablation studies are conducted to show the effectiveness of the proposed objectives and the encoder-decoder-encoder architecture for ESAD.

Our main contribution is summarized as follows: 
\begin{itemize}[leftmargin=*]
    \setlength\itemsep{0em}
    \vspace{-5pt}
    \item We introduce a KL-divergence based objective for semi-supervised anomaly detection and show that it can be relaxed and decomposed into mutual information and entropy related objectives, which formulates the AD objective with information-theoretical terms.
    \vspace{-3pt}
    \item To achieve end-to-end training, we propose an encoder-decoder-encoder architecture to simultaneously optimize the two contradictory factors, mutual information and entropy.
    \vspace{-3pt}
    \item The proposed method outperforms state-of-the-arts on multiple AD benchmarks.
\end{itemize}

\vspace{-16pt}
\section{Related Works}
\vspace{-5pt}
\noindent\textbf{Unsupervised Anomaly Detection.}
The vital challenge of unsupervised AD is that the training dataset contains only normal data. One-class classification based approaches tended to depict normal data directly with statistical approaches~\cite{Eskin2000Anomaly,scholkopf2001estimating,Xu2012Robust,Rahmani2017Coherence,ruff2018deep}. Self-supervised based approaches remedied the lack of supervision by introducing different self-supervisions, where the model was trained to optimize a self-supervised task, and then normal data can be separated with the assumption that anomalous data will perform differently. In this domain, reconstruction~\cite{Sakurada2014Anomaly,an2015variational,xia2015learning,nicolau2016hybrid,schlegl2017unsupervised,zong2018deep,deecke2018image,Sabokrou2018Adversarially,Akcay2018,gong2019memorizing} is the most popular self-supervision. Some approaches introduced other self-supervisions, \emph{e.g.}, \cite{golan2018deep} applied dozens of image geometric transforms for transformation classification, and \cite{fye2020ARNet} proposed a restoration framework to further improve the feature learning.

\noindent\textbf{Semi-supervised Anomaly Detection.}
Since classical semi-supervised approaches~\cite{kingma2014semi,rasmus2015semi,odena2016semi,dai2017good,oliver2018realistic} are inappropriate and hardly detect new and unknown anomalies due to the cluster assumption~\cite{chapelle2009semi}, many semi-supervised approaches are still grounded on the unsupervised learning paradigm~\cite{gornitz2013toward}. Along this line, Deep SSAD~\cite{gornitz2013toward} has been studied recently in specific contexts such as videos~\cite{Kiran2018An}, network intrusion detection~\cite{min2018ids}, or specific neural network architectures~\cite{ergen2017unsupervised}. Deep SAD~\cite{SAD}, a general method based on deep SVDD~\cite{ruff2018deep}, built upon the Infomax principle, where the training processes are consist of two stages. TLSAD~\cite{TLSAD} further consolidated the model's discriminative power with a transfer learning framework, which relied on an additional large-scale reference dataset for the model training.

\noindent\textbf{Anomaly Detection on Medical Images} 
is an important application but rarely considered in deep anomaly detection literature. \cite{zhou2020encoding} proposed P-Net for anomaly detection in retinal images by leveraging the specific relation between the image texture and the regular structure of retinal images, which is hard to generalize to other medical data. \cite{tuluptceva2020anomaly} relied on the classical autoencoder approach with a re-designed training pipeline to handle high-resolution, complex images. \cite{zhang2020viral} proposed a confidence-aware anomaly detection model for detecting viral pneumonia with in-house data. In this paper, we conduct experiments on some well-organized and open-source medical image datasets \cite{bejnordi2017diagnostic,wang2017chestx}.

%\vspace{-10pt}
\section{End-to-end Semi-supervised Anomaly Detection}
%\vspace{-3pt}

Given the input space $\mathcal{X}$ consisting of normal data $\mathcal{X}_N$ and anomalous data $\mathcal{X}_A$, where $\mathcal{X} = \mathcal{X}_N \cup \mathcal{X}_A$. For semi-supervised anomaly detection (AD), we are given $n$ unlabeled samples $\mathbf{x}^u_1, \cdots ,\mathbf{x}^u_n\in\mathcal{X}$ and $m$ labeled samples $(\mathbf{x}^l_1,y_1), \cdots , (\mathbf{x}^l_m, y_m) \in \mathcal{X} \times \mathcal{Y}$ with $\mathcal{Y} = \{-1,1\}$ where $y = 1$ denotes normal samples and $y = -1$ denotes anomalous samples. We assume $m \ll n$. Suppose the output space is $\mathcal{Z}$, the goal of AD is to find $f_\theta: \mathcal{X} \rightarrow \mathcal{Z}$, parameterized by $\theta$, that leads to the maximum distance between normal and anomalous data.

Targeting semi-supervised anomaly detection, we attempt to explore an objective function based on Kullback–Leibler (KL) divergence. Let $X$ and $Z$ be variables sampled from $\mathcal{X}$ and $\mathcal{Z}$, respectively.
Denote the joint distribution of data and latent representations for normal and anomalous data as $p_N(X, Z)$ and $p_A(X, Z)$, respectively, and the objective function for semi-supervised AD is then formulated as: $\max \limits_{\theta}~\operatorname{KL}\left[p_{N}(X,Z) \|p_{A}(X,Z) \right].$
Here $p_N(X, Z)$ can be approximately estimated using the labeled normal samples and the large numbers of unlabeled data, with the widely adopted assumption for AD that almost all unlabeled data are normal~\cite{gornitz2013toward,ruff2018deep,Sabokrou2018Adversarially,perera2019ocgan,zhang2020p,SAD}. On the contrary, it is impossible to estimate $p_A(X, Z)$ due to the extremely limited labeled instances. We here introduce another distribution, $p_{\tilde{A}}(X, Z)$, and reformulated the objective function as follows:
\begin{equation}\label{eq:KL obj}
\setlength{\abovedisplayskip}{2pt}
\setlength{\belowdisplayskip}{2pt}
\small
\max \limits_{\theta}~\operatorname{KL}\left[p_{N}(X,Z) \|p_{A}(X,Z) \right]-\operatorname{KL}\left[p_{\tilde{A}}(X,Z) \|p_{A}(X,Z) \right], 
\end{equation}
where $p_{\tilde{A}}(X, Z)$ can be estimated by the limited labeled anomalous data. With this objective function, we attempt to simultaneously (i) maximize KL divergence between the normal class and the anomalous class and (ii) minimize the KL divergence between the labeled anomalous class and the real anomalous class. 
Considering that it is impossible to estimate $p_A$,
we decompose the KL term $\operatorname{KL}\left[p_{N}(X,Z) \|p_{A}(X,Z) \right]$ as follows:
\begin{equation}
\small
\begin{split}\label{eq:LB_KL_reformulate}
\setlength{\abovedisplayskip}{2pt}
\setlength{\belowdisplayskip}{2pt}
&\operatorname{KL}\left[p_N(X, Z)|| p_A(X, Z) \right] = ~\mathbb{E}_{p_N(X, Z)}\left[ \log \frac{p_N(X, Z)}{p_A(X, Z)} \right]\\ 
=&~\mathbb{E}_{p_N(X, Z)} \left[\log( \frac{p_N(Z|X)}{p_N(Z)} \cdot {p_N(Z)} \cdot \frac{1} {p_A(Z|X)} \cdot \frac{p_N(X)}{p_A(X)}) \right]\\
=& ~\emph{I}(X_N,Z_N) - H(Z_N) +\mathbb{E}_{p_N(X)}\left[ H(p_N(Z|X),p_A(Z|X))\right]  + \operatorname{KL}\left[p_N({X})||p_A({X}) \right],
\end{split}
\end{equation}
where $I(\cdot, \cdot)$ is the mutual information, $H(\cdot)$ is the entropy, and $H(\cdot, \cdot)$ is the cross entropy.
With the non-negativity of the third and fourth terms (see the supplementary material for the proof), we get a lower bound to Eq.~\eqref{eq:LB_KL_reformulate}: $\operatorname{KL}\left[p_N(X, Z)|| p_A(X, Z) \right] \geq \emph{I}(X_N,Z_N) - \emph{H}(Z_N)$.
Similarly, $\operatorname{KL}\left[p_{\tilde{A}}(X,Z)~\|p_{A}(X,Z) \right]$ is approximated with $\emph{I}(X_{\scriptsize\tilde{A}},Z_{\tilde{A}})- \emph{H}(Z_{\tilde{A}})$.
The final objective function is thus formulated as:
\begin{equation}
\small
\setlength{\abovedisplayskip}{2pt}
\setlength{\belowdisplayskip}{2pt}
\begin{split}\label{eq:final}
\max \limits_{\theta}~\{[\emph{I}(X_N,Z_N)-\emph{I}(X_{\tilde{A}},Z_{\tilde{A}})] - [\emph{H}(Z_N)-\emph{H}(Z_{\tilde{A}})]\}.
\end{split}
\end{equation}
Note that this objective function is coincidentally similar to that of Deep SAD~\cite{SAD}, by optimizing on both the mutual information and entropy. However, the objective function here differs from~\cite{SAD} in that: (i) we start with a KL based formulation and derive equal weights for the mutual information and entropy, while for Deep SAD, the entropy is introduced as regularization with a coefficient $\beta$; (2) the mutual information for our paper involves different directions of optimizations for normal and anomalous data, while Deep SAD treats the normal and anomalous data the same in maximizing the mutual information. In our formulation, the optimizations of mutual information and entropy are integral parts of the single anomaly detection objective function and hence need to be optimized simultaneously.

%\subsection{ESAD} 
\noindent \textbf{Architecture.} We follow~\cite{SAD} and employ an autoencoder to optimize the mutual information $I(X,Z)$. 
To resolve the contradiction between mutual information and entropy and achieve end-to-end training, different from the straightforward solution by directly introducing two independent encoders~\cite{TLSAD}, we propose to append an additional encoder to the autoencoder and introduce an encoder-decoder-encoder architecture, where the first encoder $Enc_1(\cdot)$ emphasizes mutual information optimization and the second encoder $Enc_2(\cdot)$ focuses on entropy optimization, and in the meanwhile, the two encoders are enforced to share similar encoding via a consistent constraint on their latent representations. The encoder-decoder-encoder architecture can be expressed as:
\begin{equation}\label{eq:pipeline}
\setlength{\abovedisplayskip}{2pt}
\setlength{\belowdisplayskip}{2pt}
\mathbf{z} =Enc_1(\mathbf{x}),~
\hat{\mathbf{x}}=Dec(\mathbf{z}),~
\hat{\mathbf{z}}=Enc_2(\hat{\mathbf{x}}),
\end{equation}                     
where $\hat{\mathbf{x}}$ is the output of the decoder, and $\mathbf{z}$ and $\hat{\mathbf{z}}$ are the latent representations from the first and second encoders, respectively. The wights for the two encoders are are not shared.

\noindent \textbf{Losses.} To optimize the two factors, \textit{i.e.}, mutual information and entropy, in Eq.~\eqref{eq:final}, we propose the corresponding losses as follows.

The optimization of mutual information is achieved with reconstruction or restoration~\cite{vincent2008extracting}. With unlabeled samples $\mathbf{x}^u_1, \cdots, \mathbf{x}^u_n$ and labeled samples $\mathbf{x}^l_1, \cdots, \mathbf{x}^l_m$, we want the autoencoder to well reconstruct the normal data but erroneously reconstruct the labeled anomalous data, thus the reconstruction likelihood is maximized for the normal data and minimized for the labeled anomalous data. A straight-forward loss definition for the anomalous data is the negative squared norm loss. However, due to its unbounded nature, it is expected to result in an ill-posed optimization problem and caused optimization to diverge~\cite{SAD}. We therefore introduce a transformation function $\phi(\cdot)$ on the input, forcing the network to reconstruct the anomalous data $\mathbf{x}$ to its transformation  $\phi(\mathbf{x})$, where $\phi(\mathbf{x}) \not= \mathbf{x}, \forall \mathbf{x} \in \mathcal{X}_A$. The transformation makes the network unable to correctly reconstruct the anomalous samples. The reconstruction loss is defined as follows:
\begin{equation}
\setlength{\abovedisplayskip}{2pt}
\setlength{\belowdisplayskip}{2pt}
\mathcal{L}_{rec-semi} = \frac{1}{n}\sum_{i = 1}^{n}\|\hat{\mathbf{x}}^u_i-\mathbf{x}^u_i\|^2 + \frac{1}{m}\sum_{j = 1}^{m}\|\hat{\mathbf{x}}^l_j-\Phi (\mathbf{x}^l_j)\|^2,
\end{equation}
where
$\Phi(\mathbf{x}^l_j) = \left\{\begin{matrix}
\mathbf{x}^l_j,  & \text{ if } ~y_j = 1,~~~\\ 
\phi(\mathbf{x}^l_j), & \text{ if } ~y_j = -1.
\end{matrix}\right.$
For the data which is functioned as a vector, $\phi(\mathbf{x}^l_j)$ can be a version adding Gaussian noise or a random permutation between various dimensions; for the image data, it can be a noised and rotated version of the original images. Besides the proposed transformation function, we also try another strategy, which enforces the model to reconstruct the labeled anomalous data to the normal data~\cite{perera2019ocgan}. But this task is too strict and difficult, especially for two types of samples that are quite different, which makes the model hard to converge.

For the entropy $H(Z)$, assuming $Z$ follows an isotropic Gaussian~\cite{cover2012elements}, $Z \sim N(\boldsymbol{\mu}, \sigma^{2} I)$ with $\sigma>0$, the entropy of $Z$ is proportional to its log-variance, \emph{i.e.}, $\emph{H}(Z) \propto \log \sigma^{2}$ (see the supplementary material for the proof). 
In this case, for $\mathbf{z}\sim p(Z)$, an $L_2$ norm can be used for the optimization of the entropy, since it minimizes the empirical variance and thus minimizes the entropy of a latent Gaussian. 
\begin{equation}
\setlength{\abovedisplayskip}{2pt}
\setlength{\belowdisplayskip}{2pt}
\mathcal{L}_{norm-semi} = \frac{1}{n}\sum_{i = 1}^{n}\|\hat{\mathbf{z}}^u_i\|_2 + \frac{1}{m}\sum_{j = 1}^{m}(\|\hat{\mathbf{z}}^l_j\|_2)^{y_j},
\end{equation}
where $y_j$$=-1$ for the labeled anomalous data while $y_j$$=1$ for the labeled normal data. This loss enforces the compacted representation for the normal data and scattered representation for the labeled anomalous data. 
Note that the inverse squared norm loss used for labeled anomalous data here is bounded from below and smooth, which are crucial properties for losses used in deep learning~\cite{Goodfellow2016}.
Compared with the SVDD loss in~\cite{SAD} where a pre-training process is necessary for initializing an additional hypersphere center, $\mathcal{L}_{norm-semi}$ does not need the pre-training, which indicates that the end-to-end training can be achieved. 

To define the consistency between the two encoders, similar to the assistant loss \cite{Akcay2018}, we resort to a consistent constraint between their corresponding latent representations:
\begin{equation}\label{eq:ass2}
\setlength{\abovedisplayskip}{2pt}
\setlength{\belowdisplayskip}{2pt}
    \mathcal{L}_{ass} = \frac{1}{n+m}\sum_{i = 1}^{n+m}\|\hat{\mathbf{z}}_i-\mathbf{z}_i\|^2.
\end{equation}

\noindent Finally, we define our training loss as follow:
\begin{equation}\label{eq:loss_final}
\setlength{\abovedisplayskip}{2pt}
\setlength{\belowdisplayskip}{2pt}
    \mathcal{L}_{semi} = \mathcal{L}_{rec-semi} + \lambda_1 \mathcal{L}_{norm-semi} + \lambda_2 \mathcal{L}_{ass},
\end{equation}
where $\lambda_1$ and $\lambda_2$ are two hyperparameters. We will further discuss the impacts of these two hyperparameters in the experiment section. To this end, we achieve end-to-end training for semi-supervised anomaly detection.

\noindent \textbf{Anomaly Score Measurement.}
We discuss how we calculate the anomaly score in the test phase. Since both the mutual information and the entropy are related to the performance of anomaly detection, we use both $\mathcal{L}_{rec-semi}$ and $\mathcal{L}_{norm-semi}$ to measure the anomaly score for the given samples, which are related to the mutual information and the entropy, respectively. 
%In the test phase, 
We calculate the reconstruction error of each input sample $\mathbf{x}$ and the value of $L_2$ norm for its representation $\hat{\mathbf{z}}$ for anomaly detection. The anomaly score is formulated as:
\begin{equation}
\setlength{\abovedisplayskip}{2pt}
\setlength{\belowdisplayskip}{2pt}
    \mathcal{S}_{test} = \|\hat{\mathbf{x}}-\mathbf{x}\|^2 + \lambda_1 \|\hat{\mathbf{z}}\|_2,
\end{equation}
where $\lambda_1$ is the same as the setting in the training process. We will further discuss the impact of $\lambda_1$ in Section~\ref{sec45}.
To the best of our knowledge, it is the first time considering both the terms of the mutual information and the entropy for the anomaly score measurement. On the contrary, most one-class classification based methods, \emph{e.g.}, OC-SVM~\cite{scholkopf2001estimating}, only consider the term of the entropy. Similarly, Deep SVDD~\cite{ruff2018deep}, Deep SAD~\cite{SAD} and TLSAD~\cite{TLSAD} also consider only the term of the entropy, since they only use the SVDD loss as the final anomaly score. Most reconstruction based methods or restoration based methods, including the vanilla AE~\cite{masci2011stacked} and ARNet~\cite{fye2020ARNet}, only consider the term of mutual information, since they only use the reconstruction or restoration loss as the anomaly score. Results show that considering both of the two terms significantly improves the performance of anomaly detection.

%\vspace{-10pt}
\section{Experiments}
%\vspace{-3pt}
In this section, we conduct substantial experiments to validate our method. The ESAD is first evaluated on multiple AD benchmark datasets, comparing with several state-of-the-arts. Then we present the respective effects of different designs through ablation study. Finally, we visualize the latent representations of ESAD through T-SNE.

%\vspace{-10pt}
\subsection{Experimental Setups}
%\vspace{-3pt}
\noindent\textbf{Datasets.}
We conduct semi-supervised anomaly detection experiments on three popular natural image datasets MNIST~\cite{lecun1998mnist}, Fashion-MNIST~\cite{xiao2017fashion} and CIFAR-10~\cite{krizhevsky2009learning}, together with six non-image classic AD datasets~\cite{Rayana2016}, all following the settings in~\cite{SAD}. 
To validate our method on real-world AD scenarios, \emph{i.e.}, with higher resolution and with more complex anomalies, we additionally conduct experiments on two medical image datasets Camelyon16~\cite{bejnordi2017diagnostic} and the NIH dataset~\cite{wang2017chestx}. For all datasets, the training and testing partitions remain as default. More details are shown in the supplementary material.

\noindent\textbf{Evaluation Protocol.} 
We quantify the model performance using the area under the Receiver Operating Characteristic (ROC) curve metric (AUC). It is commonly adopted as performance measurement in anomaly detection (AD) tasks.

\noindent\textbf{Model Configuration.}
For ESAD, the architecture of the autoencoder and the data preprocessing for the image dataset is aligned with~\cite{fye2020ARNet}. Different from Deep SAD, which uses different networks for each dataset, ESAD uses the same autoencoder network since it is robust enough. For non-image classic AD datasets, we use the autoencoder network aligned with~\cite{SAD}. The hyperparameter $\lambda_1$ and $\lambda_2$ are set to 1 as default. We give the full details in the supplementary material.

\renewcommand \arraystretch{0.7}
\begin{table}[t]
\centering
\caption{Results of anomaly detection on natural image datasets, where we increase the ratio of labeled anomalies $\gamma_{l}$ in the training set. We report the avg. AUC in \% with st. dev. computed over 90 experiments at various $\gamma_{l}$. Results of SSAD Hybrid, SS-DGM and Deep SAD are borrowed from~\cite{SAD}. Results of TLSAD are borrowed from~\cite{TLSAD}.}
\label{tal:1}
\footnotesize
\setlength{\tabcolsep}{2.5pt}{
\begin{tabular}{C{1.3cm}C{0.6cm}C{2.2cm}C{1.7cm}C{1.8cm}C{1.1cm}C{1.6cm}}
\toprule
Data & $\gamma_l$ & \makecell[c]{SSAD Hybrid\\~\cite{gornitz2013toward}} & \makecell[c]{SS-DGM\\~\cite{kingma2014semi}} & \makecell[c]{Deep SAD\\~\cite{SAD}} & \makecell[c]{TLSAD\\~\cite{TLSAD}} & \makecell[c]{ESAD\\ (ours)}\\
\cmidrule(lr){1-1} \cmidrule(lr){2-2} \cmidrule(lr){3-3} \cmidrule(lr){4-4} \cmidrule(lr){5-5} \cmidrule(lr){6-6} \cmidrule(lr){7-7}
& .00 & 96.3 $\pm$ 2.5 & - & 92.8 $\pm$ 4.9 & - & \textbf{98.5} $\pm$ \textbf{1.3}\\
 & .01 & 96.8 $\pm$ 2.3 & 89.9 $\pm$ 9.2 & 96.4 $\pm$ 2.7 & 94.1 & \textbf{99.2} $\pm$ \textbf{0.7} \\
MNIST  & .05 & 97.4 $\pm$ 2.0 & 92.2 $\pm$ 5.6 & 96.7 $\pm$ 2.4 & 96.9 & \textbf{99.4} $\pm$ \textbf{0.3}\\
 & .10 & 97.6 $\pm$ 1.7 & 91.6 $\pm$ 5.5 & 96.9 $\pm$ 2.3 & 97.7 & \textbf{99.5} $\pm$ \textbf{0.4} \\
 & .20 & 97.8 $\pm$ 1.5 & 91.2 $\pm$ 5.6 & 96.9 $\pm$ 2.4 & 98.3 & \textbf{99.6} $\pm$ \textbf{0.3} \\
\cmidrule(lr){1-1} \cmidrule(lr){2-2} \cmidrule(lr){3-3} \cmidrule(lr){4-4} \cmidrule(lr){5-5} \cmidrule(lr){6-6} \cmidrule(lr){7-7}
 & .00 & 91.2 $\pm$ 4.7 & & 89.2 $\pm$ 6.2 & - & \textbf{94.0} $\pm$ \textbf{4.5}\\
 & .01 & 89.4 $\pm$ 6.0 & 65.1 $\pm$ 16.3 & 90.0 $\pm$ 6.4 & 88.4 & \textbf{95.3} $\pm$ \textbf{4.2}\\
F-MNIST  & .05 & 90.5 $\pm$ 5.9 & 71.4 $\pm$ 12.7 & 90.5 $\pm$ 6.5 & 91.4 & \textbf{95.6} $\pm$ \textbf{4.1}\\
 & .10 & 91.0 $\pm$ 5.6 & 72.9 $\pm$ 12.2 & 91.3 $\pm$ 6.0 & 92.0 & \textbf{95.8} $\pm$ \textbf{4.0}\\
 & .20 & 89.7 $\pm$ 6.6 & 74.7 $\pm$ 13.5 & 91.0 $\pm$ 5.5 & 93.2 & \textbf{95.9} $\pm$ \textbf{4.0}\\
\cmidrule(lr){1-1} \cmidrule(lr){2-2} \cmidrule(lr){3-3} \cmidrule(lr){4-4} \cmidrule(lr){5-5} \cmidrule(lr){6-6} \cmidrule(lr){7-7}
 & .00 & 63.8 $\pm$ 9.0 & & 60.9 $\pm$ 9.4 & - & \textbf{78.8} $\pm$ \textbf{6.5}\\
  & .01 & 70.5 $\pm$ 8.3 & 49.7 $\pm$ 1.7 & 72.6 $\pm$ 7.4 & 74.4 & \textbf{83.7} $\pm$ \textbf{6.4}\\
 CIFAR-10 & .05 & 73.3 $\pm$ 8.4 & 50.8 $\pm$ 4.7 & 77.9 $\pm$ 7.2 & 80.0 & \textbf{86.9} $\pm$ \textbf{6.8} \\
 & .10 & 74.0 $\pm$ 8.1 & 52.0 $\pm$ 5.5 & 79.8 $\pm$ 7.1 & 84.8 & \textbf{87.8} $\pm$ \textbf{6.7} \\
 & .20 & 74.5 $\pm$ 8.0 & 53.2 $\pm$ 6.7 & 81.9 $\pm$ 7.0 & 86.3 & \textbf{88.5} $\pm$ \textbf{6.9} \\
\bottomrule
\end{tabular}}
\vspace{-12pt}
\end{table}

%\vspace{-10pt}
\subsection{Experiments on Natural Images}
\vspace{-3pt}
\noindent\textbf{Competing Methods.}
We consider several semi-supervised anomaly detection state-of-the-arts, including SSAD~\cite{gornitz2013toward}, SS-DGM~\cite{kingma2014semi}, Deep SAD~\cite{SAD} and TLSAD~\cite{TLSAD} as baselines.
Following~\cite{SAD}, as~\cite{gornitz2013toward} is sensitive to hyperparameters, SSAD Hybrid here uses a subset $(10\%)$ of the test set for hyperparameter selection to establish a strong baseline.
More details for these baseline methods are shown in the supplementary material.

\noindent\textbf{Experiment Settings.}
For a dataset with $C$ classes, we conduct a batch of $C$ experiments respectively with each of the $C$ classes set as the normal class once. We then evaluate performance on an independent test set, which contains samples from all classes, including normal and anomalous data. 

\noindent\textbf{Comparison with State-of-the-art Methods.}
The effectiveness of adding labeled anomalies during training is investigated. By adding labeled anomalous samples $\mathbf{x}_{1}, \ldots, \mathbf{x}_{m}$ to the training set, we increase the ratio of labeled training data $\gamma_{l}=m /(n+m)$. The number of anomaly classes included in the labeled training data is set as 1, \emph{i.e.}, there are still eight unseen classes at testing time. We iterate this training set generation process and report the average results over the ten kinds of normal classes $\times$ nine labeled anomalous classes, \emph{i.e.}, over 90 experiments per labeled ratio $\gamma_{l}$. The corresponding results are shown in Table~\ref{tal:1}. On all involved datasets, results present that the average AUCs of ESAD outperform all other methods, including TLSAD which utilizes a large-scale additional dataset (ImageNet~\cite{russakovsky2015imagenet}) as the reference data for the model training. Results when $\gamma_l > 0$ are much better than the results when $\gamma_l = 0$, showing the effectiveness of the semi-supervised training scheme.

\subsection{Experiments on Medical Images}\label{sec:medical}
Medical images, such as H\&E stained images, X-ray, etc., have extremely high resolution compared to natural images. In addition, the patient's lesions may only occupy a small part of the entire image, which brings great challenges to AD. To validate the AD performance of ESAD on real-world AD scenarios, we examined two challenging medical problems with different image characteristics and abnormality appearance, \emph{i.e.}, Camelyon16 \cite{bejnordi2017diagnostic} and chest X-rays in NIH~\cite{wang2017chestx}. 
We consider several state-of-the-arts, including DAOL \cite{tang2019deep}, DGEO \cite{golan2018deep}, PIAD \cite{tuluptceva2019perceptual}, DIF \cite{ouardini2019towards} Deep SAD \cite{SAD}, and DPA \cite{tuluptceva2020anomaly}. Note that for
%the unsupervised baselines
\cite{golan2018deep,tuluptceva2019perceptual,ouardini2019towards}, anomalous samples in the training set are used for the validation. We re-train Deep SAD \cite{SAD} with the same encoder and decoder network as ESAD to obtain a better baseline.

\renewcommand \arraystretch{0.95}
\begin{table}[t]
\centering
\caption{Performance of anomaly detection methods on medical image datasets. We report the avg. AUC in \% with st. dev. computed over 3 runs.}
\label{tal:medical}
\footnotesize
\setlength{\tabcolsep}{7.5pt}{
\begin{tabular}{C{2.1cm}C{1.7cm}C{1.7cm}C{1.7cm}C{1.7cm}}
\toprule
Method & Cam.16 & NIH (a sub.) & NIH (PA) & NIH (AP)\\
\cmidrule(lr){1-1} \cmidrule(lr){2-2} \cmidrule(lr){3-3} \cmidrule(lr){4-4} \cmidrule(lr){5-5}
DAOL \cite{tang2019deep} & - & 80.5 $\pm$ 2.1 & - & -\\
DGEO \cite{golan2018deep} & 45.9 $\pm$ 2.1 & 85.3 $\pm$ 1.0 & 63.6 $\pm$ 0.6 & 54.4 $\pm$ 0.6\\
PIAD \cite{tuluptceva2019perceptual} & 89.5 $\pm$ 0.6 & 87.3 $\pm$ 0.9 & 68.7 $\pm$ 0.5 & 58.6 $\pm$ 0.3\\
DIF \cite{ouardini2019towards} & 90.6 $\pm$ 0.3 & 85.3 $\pm$ 0.4 & 47.2 $\pm$ 0.4 & 56.1 $\pm$ 0.2\\
Deep SAD \cite{SAD} & 92.1 $\pm$ 0.4 & 90.9 $\pm$ 0.2 & 51.9 $\pm$ 0.8 & 59.8 $\pm$ 0.1\\
DPA \cite{tuluptceva2020anomaly} & 93.4 $\pm$ 0.3 & 92.6 $\pm$ 0.2 & \textbf{70.8 $\pm$ 0.1} & 58.5 $\pm$ 0.0\\
ESAD (ours) & \textbf{96.8 $\pm$ 0.4} & \textbf{94.6 $\pm$ 0.4} & 68.9 $\pm$ 0.2 & \textbf{60.1 $\pm$ 0.2}\\
\bottomrule
\end{tabular}}
\vspace{-8pt}
\end{table}

\renewcommand \arraystretch{0.95}
\begin{table*}[t]
\centering
\caption{Results on classic anomaly detection benchmark datasets with a ratio of labeled anomalies of $\gamma_{l}=0.01$. We report the avg. AUC in \% with st. dev. computed over 10 seeds.}
\label{tal:4}
\footnotesize
\setlength{\tabcolsep}{7.7pt}{
\begin{tabular}{C{1.4cm}C{1.5cm}C{1.6cm}C{1.5cm}C{1.4cm}C{1.4cm}}
\toprule
Data & \makecell[c]{Deep\\ SVDD~\cite{ruff2018deep}} & \makecell[c]{SSAD\\ Hybrid~\cite{gornitz2013toward}} & \makecell[c]{SS-DGM\\\cite{kingma2014semi}} & \makecell[c]{Deep\\ SAD~\cite{SAD}} & \makecell[c]{ESAD \\ (ours)}\\
\cmidrule(lr){1-1} \cmidrule(lr){2-2} \cmidrule(lr){3-3} \cmidrule(lr){4-4} \cmidrule(lr){5-5} \cmidrule(lr){6-6} 
arrhythmia & 74.6 $\pm$ 9.0 & 78.3 $\pm$ 5.1 & 50.3 $\pm$ 9.8 & 75.9 $\pm$ 8.7 & \textbf{85.2} $\pm$ \textbf{2.9}\\
cardio & 84.8 $\pm$ 3.6 & 86.3 $\pm$ 5.8 & 66.2 $\pm$ 14.3 & 95.0 $\pm$ 1.6 & \textbf{98.8} $\pm$ \textbf{0.5}\\
satellite & 79.8 $\pm$ 4.1 & 86.9 $\pm$ 2.8 & 57.4 $\pm$ 6.4 & 91.5 $\pm$ 1.1 & \textbf{92.5} $\pm$ \textbf{0.7}\\
satimage-2 & 98.3 $\pm$ 1.4 & 96.8 $\pm$ 2.1 & 99.2 $\pm$ 0.6 & \textbf{99.9} $\pm$ \textbf{0.1} & \textbf{99.9} $\pm$ \textbf{0.1}\\
shuttle & 86.3 $\pm$ 7.5 & 97.7 $\pm$ 1.0 & 97.9 $\pm$ 0.3 & 98.4 $\pm$ 0.9 & \textbf{99.1} $\pm$ \textbf{1.1}\\
thyroid & 72.0 $\pm$ 9.7 & 95.3 $\pm$ 3.1 & 72.7 $\pm$ 12.0 & 98.6 $\pm$ 0.9 & \textbf{99.6} $\pm$ \textbf{0.2}\\
\bottomrule
\end{tabular}}
\vspace{-10pt}
\end{table*}

\noindent\textbf{Anomaly Detection on Chest X-Rays.} NIH images without any disease marker were considered normal. Pulmonary and cardiac abnormalities in this dataset are all considered anomalous. Following~\cite{tang2019deep,tuluptceva2020anomaly}, we split the dataset into two sub-datasets having only posteroanterior (PA) or anteroposterior (AP) projections. The labeled anomalous samples contain only the most frequent disease (`Infiltration') out of fourteen possibilities and there are still thirteen unseen possibilities of anomalies in the test set. We also experiment on a subset containing clearer normal/anomalous cases \cite{tang2019deep}. Default preprocessing of chest X-rays involved a $768\times 768$ central crop and resize to $64\times 64$. As shown in Table~\ref{tal:medical}, the anomaly detection performance of ESAD outperforms all state-of-the-arts on the clearer subset \cite{tang2019deep} and AP subset. DPA~\cite{tuluptceva2020anomaly} performs better than ESAD on the subset of PA. Note that DPA uses a higher resolution version of the images ($256 \times 256$) for validation, so it has a greater advantage compared with other methods.

\noindent\textbf{Metastases Detection in Digital Pathology.}
For the Camelyon16 Challenge~\cite{bejnordi2017diagnostic}, we sample the Vahadane-normalized~\cite{vahadane2016structure} $64\times 64$ tiles from the fully normal slides with magnification of $10\times$, and treat these as normal. Tiles with metastases are treated as anomalous. It contains 7612 normal and 200 anomalous training images, and 4000 (normal) + 817 (anomalous) images for the test. As shown in Table~\ref{tal:medical}, the anomaly detection performance of ESAD outperforms all state-of-the-art methods. 

%\vspace{-10pt}
\subsection{Experiments on Classic Anomaly Detection Benchmark Datasets}
%\vspace{-3pt}
We examine the performance of the various methods on well-established classic AD benchmark datasets \cite{Rayana2016} with $\gamma_{l}=0.01$. Networks of both the encoder and the decoder are aligned with~\cite{SAD}. The corresponding results are shown in Table~\ref{tal:4}. Comparing with other state-of-the-arts, ESAD shows the highest AUCs and stability. It shows that unlike other deep approaches~\cite{ergen2017unsupervised,Kiran2018An,min2018ids,deecke2018image,golan2018deep}, ESAD is not domain or data-type specific. 

\renewcommand \arraystretch{0.75}
\begin{table}[t]
\centering
\caption{Ablation study on different designs of architecture and loss functions for ESAD on two different datasets. We report the avg. AUC in \% with st. dev. computed over 10 seeds.}
\label{tal:ablation}
\footnotesize
\setlength{\tabcolsep}{2.5pt}{
\begin{tabular}{C{3.0cm}|C{1.5cm}C{1.5cm}C{1.5cm}|C{1.5cm}C{1.5cm}}
\toprule
\multirow{2}{*}{Architecture} & \multicolumn{3}{c|}{Loss Functions} & \multicolumn{2}{c}{\multirow{2}{*}{Dataset}}\\
\cmidrule(lr){2-4}
& \multirow{2}{*}{$\mathcal{L}_{ass}$} & \multicolumn{2}{c|}{$\mathcal{L}_{rec-semi}$} & \multicolumn{2}{c}{}\\ 
\cmidrule(lr){1-1} \cmidrule(lr){3-4} \cmidrule(lr){5-6}
Encoder-decoder-encoder & & Gaussian & Permutation & satellite & \multicolumn{1}{c}{cardio} \\
\cmidrule(lr){1-1} \cmidrule(lr){2-2} \cmidrule(lr){3-3} \cmidrule(lr){4-4}
\cmidrule(lr){5-5} \cmidrule(lr){6-6}
\ding{55}&\ding{55}&\ding{55}&\ding{55}&87.9 $\pm$ 1.7 & 96.5 $\pm$ 1.1\\
\checkmark &\ding{55}&\ding{55}&\ding{55} & 90.0 $\pm$ 1.2 & 97.2 $\pm$ 0.9\\
\checkmark &\checkmark&\ding{55}&\ding{55} & 90.4 $\pm$ 1.1 & 97.9 $\pm$ 1.0\\
\checkmark &\checkmark&\checkmark &\ding{55}& 92.0 $\pm$ 1.1 & 98.2 $\pm$ 0.6\\
\checkmark &\checkmark&\ding{55}&\checkmark& 92.5 $\pm$ 1.0 & 98.6 $\pm$ 0.6\\
\checkmark &\checkmark&\checkmark & \checkmark & 92.5 $\pm$ 0.7 & 98.8 $\pm$ 0.5\\
\bottomrule
\end{tabular}}
\vspace{-5pt}
\end{table}

\renewcommand \arraystretch{0.75}
\begin{table}[t]
\centering
\caption{Ablation study on shallow and deep networks, for both the encoder and the decoder. `Shallow' is a LeNet-type network utilized in~\cite{SAD}. `Deep' is the network utilized in ESAD. We report the avg. AUC in \% with st. dev. computed over 90 experiments at various $\gamma_{l}$ on F-MNIST. Results with * are lower than expected because of the model collapse problem for Deep SAD under the small labeled anomalies ratio.}
%Results of anomaly detection on natural image datasets, where we increase the ratio of labeled anomalies $\gamma_{l}$ in the training set. We report the avg. AUC in \% with st. dev. computed over 90 experiments at various ratios $\gamma_{l}$.
\label{tal:backbone}
%\scriptsize
\footnotesize
%\small
\setlength{\tabcolsep}{1.5pt}{
\begin{tabular}{C{1.4cm}C{2.3cm}C{1.6cm}C{1.6cm}C{1.6cm}C{1.6cm}C{1.6cm}}
\toprule
Network & Method & $\gamma_l=0.0$ & $\gamma_l=0.01$ &$\gamma_l=0.05$ &$\gamma_l=0.1$ & $\gamma_l=0.2$ \\
\cmidrule(lr){1-1} \cmidrule(lr){2-2} \cmidrule(lr){3-3} \cmidrule(lr){4-4} \cmidrule(lr){5-5} \cmidrule(lr){6-6} \cmidrule(lr){7-7}
\multirow{2}{*}{Shallow} & \makecell[c]{Deep SAD~\cite{SAD}} &
89.2 $\pm$ 6.2 & 
90.0 $\pm$ 6.4 &
90.5 $\pm$ 6.5 &
91.3 $\pm$ 6.0 & 
91.0 $\pm$ 5.5 \\
& \makecell[c]{ESAD (ours)} &
93.6 $\pm$ 4.5 & 
94.9 $\pm$ 4.2 &
95.3 $\pm$ 4.2 &
95.4 $\pm$ 4.1 & 
95.5 $\pm$ 4.1 \\
\cmidrule(lr){1-1} \cmidrule(lr){2-2} \cmidrule(lr){3-3} \cmidrule(lr){4-4} \cmidrule(lr){5-5} \cmidrule(lr){6-6} \cmidrule(lr){7-7}
\multirow{2}{*}{Deep} & \makecell[c]{Deep SAD~\cite{SAD}} &
72.5 $\pm$ 7.0* & 
87.0 $\pm$ 8.7* &
90.3 $\pm$ 6.4 &
91.8 $\pm$ 7.7 & 
92.0 $\pm$ 7.0 \\
& \makecell[c]{ESAD (ours)} &
94.0 $\pm$ 4.5 & 
95.3 $\pm$ 4.2 &
95.6 $\pm$ 4.1 &
95.8 $\pm$ 4.0 & 
95.9 $\pm$ 4.0 \\
\bottomrule
\end{tabular}}
\vspace{-5pt}
\end{table}

\subsection{Ablation Study}\label{sec45}
The model architecture and different losses for ESAD are discussed in Table~\ref{tal:ablation} through ablation studies. Experiments are conducted on two datasets, \emph{i.e.}, cardio and satellite. Firstly, for the model architecture, results show that without the encoder-decoder-encoder architecture, ESAD with vanilla autoencoder shows relatively low and unstable AUCs (the entropy loss is conducted on the first encoder in this case). Secondly, ablation studies on two proposed losses, \emph{i.e.}, $\mathcal{L}_{ass}$ and $\mathcal{L}_{rec-semi}$, show impressive results. Comparing with vanilla reconstruction loss, $\mathcal{L}_{rec-semi}$ utilizes two transformations for changing the supervisions of labeled anomalous data. Without these transformations, it degrades to the vanilla reconstruction loss where the original data are used as the reconstruction supervisions, leading to relatively lower AUCs. Note that the entropy loss should always be utilized in all experiments since it is highly relative to the anomaly score measurement, but its importance can be shown towards the following discussions for the hyperparameters. 

\begin{figure}[t]
\centering
\begin{minipage}[t]{0.47\textwidth}
\centering
\includegraphics[width=4.7cm]{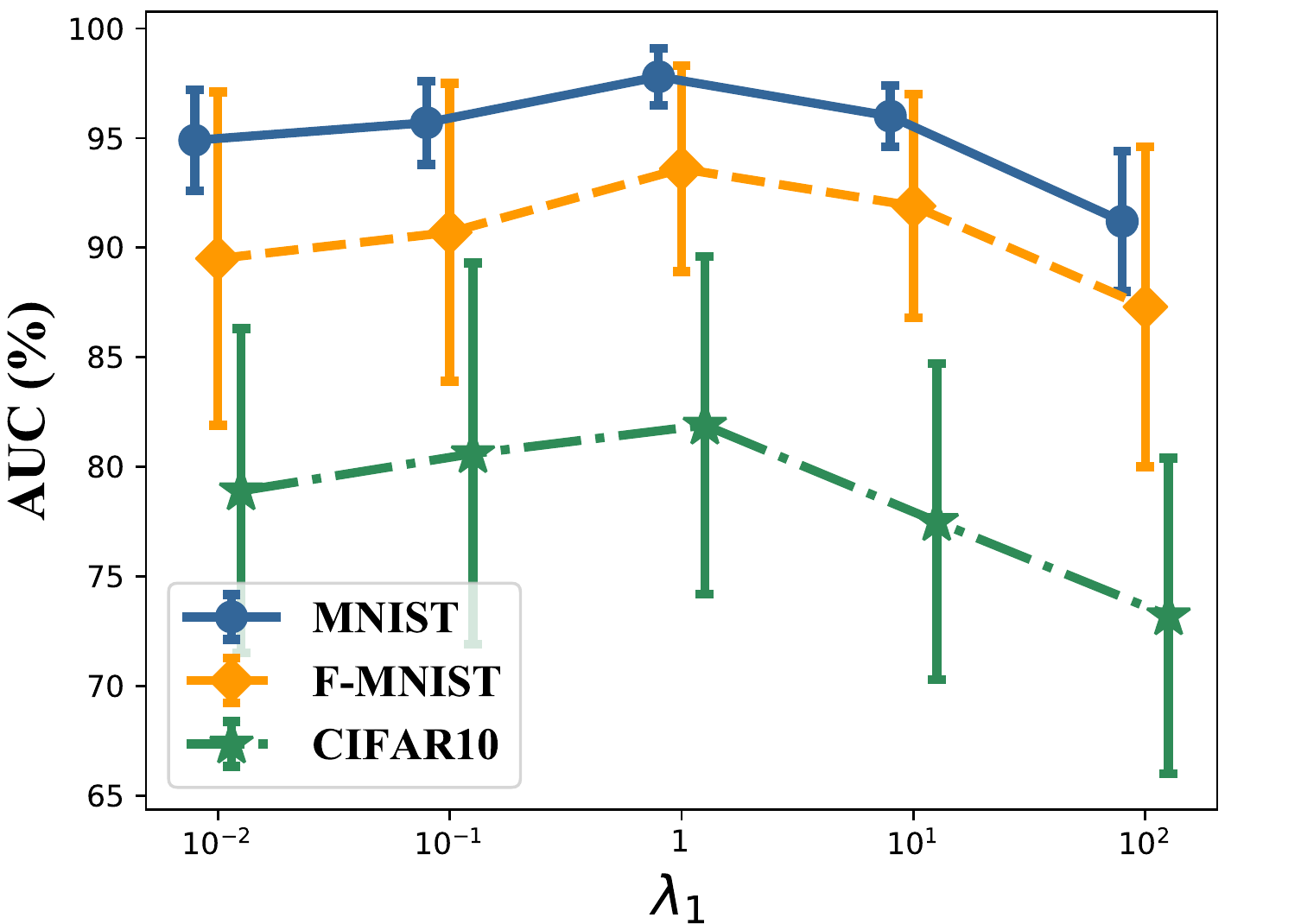}
\vspace{-8pt}
\caption{\label{img:lamda1}ESAD sensitivity analysis w.r.t. $\lambda_1$. We report avg. AUC in \% with st. dev. over 90 experiments. Best viewed in color.}
\end{minipage}
\begin{minipage}[t]{0.2\textwidth}
\end{minipage}
\begin{minipage}[t]{0.47\textwidth}
\centering
\includegraphics[width=5.2cm]{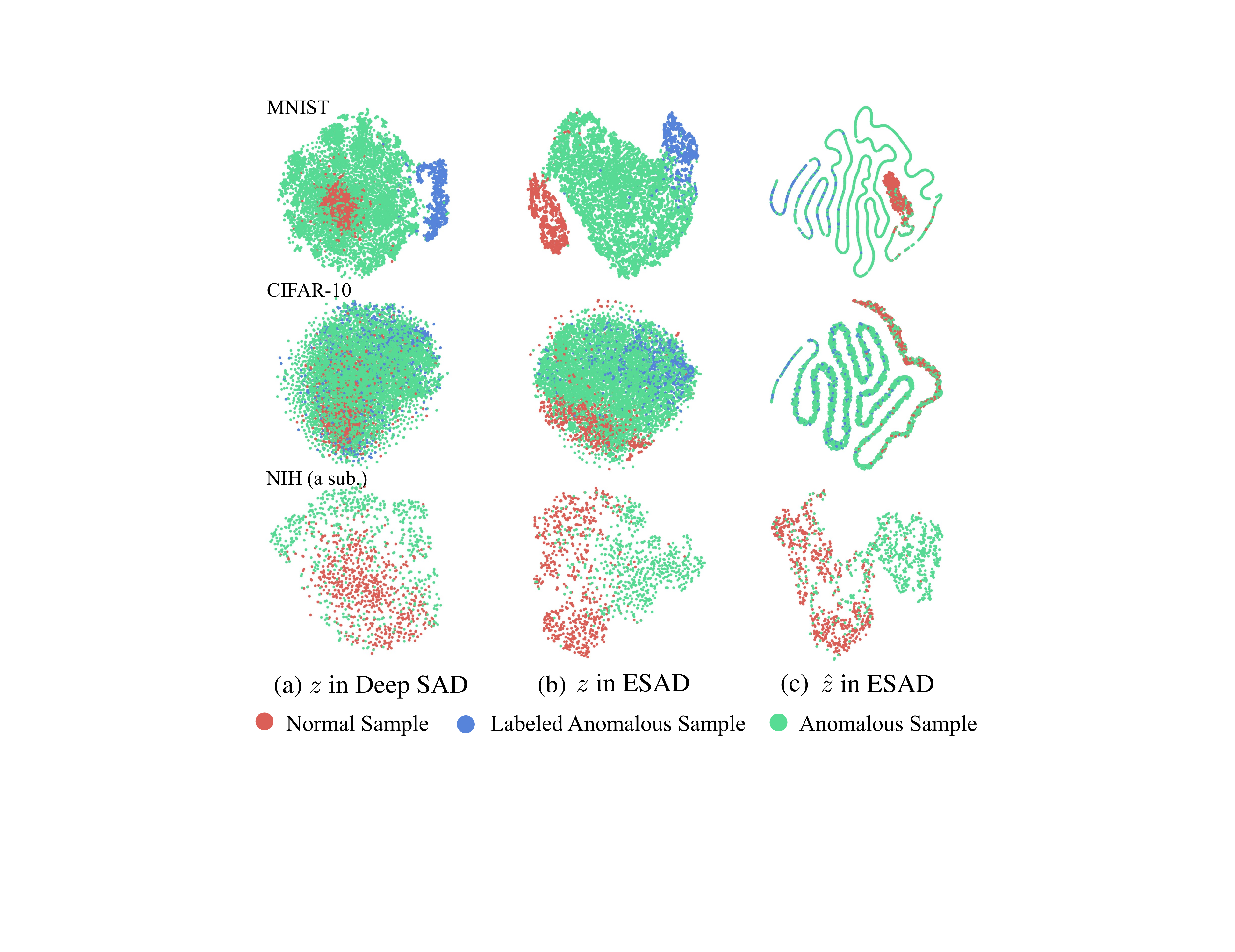}
\vspace{-5pt}
\caption{\label{img:tsne}T-SNE visualization of latent representations on MNIST (top), CIFAR-10 (middle) and NIH (bottom).}
\end{minipage}
\vspace{-5pt}
\end{figure}

Then we analyze the influence of the network choices. For the natural image datasets, Deep SAD~\cite{SAD} uses different LeNet-type networks for each dataset. ESAD does not follow~\cite{SAD} and uses the same network for all datasets. To show the influence of network choices, as shown in Table~\ref{tal:backbone}, we experiment with different networks and show that: i) ESAD is robust to different networks, with a performance gap of 0.3\% - 0.4\% between using the shallow or deep network on F-MNIST, while Deep SAD encounters model collapse with certain networks.  ii) ESAD outperforms Deep SAD for both shallow and deep networks. Results on more datasets are shown in the supplemental material.

We further analyze the sensitivity of the hyperparameters of ESAD. According to Eq.~\eqref{eq:loss_final}, $\lambda_1$ has a certain impact on the performance of semi-supervised AD. The larger $\lambda_1$ means more attention is paid to the entropy, while the smaller $\lambda_1$ pays more attention to the mutual information. Figure~\ref{img:lamda1} shows the ESAD performance with different $\lambda_1$. The results show that the best AUC can be obtained when $\lambda_1$ is set as 1 in all datasets. When $\lambda_1$ is relatively too small or too large, relatively poor AD performance will be achieved. Fortunately, the relationship between AUCs and the $\lambda_1$ presents the same pattern in all datasets, which means that when changing datasets, we may not need to spend too many resources on the adjustment of $\lambda_1$. For $\lambda_2$, we found through experiments that modifying $\lambda_2$ has a relatively small impact. We thus always set $\lambda_2$ as 1. More details are shown in the supplementary material.

\vspace{-10pt}
\subsection{Visualization Analysis}\label{sec:vis}
We show that the latent representations extracted by ESAD can better be used to distinguish samples of different categories through T-SNE~\cite{maaten2008visualizing} analysis. We conduct experiments on MNIST, CIFAR-10 and the medical image dataset NIH. Figure~\ref{img:tsne}~(a) shows the results using latent representations extracted by Deep SAD. Figure~\ref{img:tsne}~(b) and (c) visualize different latent representations, \emph{i.e.}, $z$ and $\hat{z}$, extracted by ESAD, which are more discriminative than the baseline. In Figure~\ref{img:tsne}~(c), $\hat{z}$ shows a more specific structure. It shows that the two latent representations have learned different information.

\vspace{-10pt}
\section{Conclusion}
\vspace{-3pt}
In this paper, we show that factors of \emph{mutual information} and \emph{entropy} constitute an integral objective function for anomaly detection. We achieve end-to-end training by proposing a novel model architecture. The proposed information theoretic framework can also be applied to more semi-supervised tasks, opening avenues for future research.

\paragraph{Acknowledgements.} 
This work is supported by the National Key Research and Development Program of China (No. 2019YFB1804304), SHEITC (No. 2018-RGZN-02046), 111 plan (No. BP0719010),  and STCSM (No. 18DZ2270700), and State Key Laboratory of UHD Video and Audio Production and Presentation.

\small
\bibliography{egbib}

\section{Supplementary Material}
\subsection{Supplementary Proofs} 

\begin{proposition}\label{rmk:eq}
If $\operatorname{KL}\left[p_{N}(X, Z)|| p_{A}(X, Z) \right]$ is maximized, then it is equivalent that \\ $\operatorname{KL}\left[p_{N}(X)|| p_{A}(X) \right]$ and $\operatorname{KL}\left[p_{N}(Z|X)|| p_{A}(Z|X) \right]$ are maximized.
\end{proposition}

\begin{proof}
The KL divergence for the joint distributions can be decomposed with the chain rule~\cite{cover2006elements}:
\begin{equation}
\small
\begin{split}\nonumber
&\operatorname{KL}\left[p_{N}(X, Z)|| p_{A}(X, Z) \right] \\
=&\mathbb{E}_{p_{N}(X, Z)}\left[ \log \frac{p_{N}(X, Z)}{p_{A}(X, Z)} \right]\\ 
=&\mathbb{E}_{p_{N}(X, Z)} \left[ \log \frac{p_{N}(X)}{p_{A}(X)} + \log \frac{p_{N}(Z|X)}{p_{A}(Z|X)} \right] \\
=&\operatorname{KL}\left[p_{N}(X)|| p_{A}(X) \right] + \mathbb{E}_{p_{N}(X)} \left[ \operatorname{KL}\left[p_{N}(Z |X)|| p_{A}(Z |X) \right] \right].
\end{split}
\end{equation}
To maximize the KL divergence for the joint distributions, it is equivalent that we maximize the KL divergence for both marginal and conditional distributions \cite{dumoulin2016adversarially}.
\end{proof}

\begin{proposition}
Let $I(X_{N}, Z_{N})$ denotes the mutual information between $X_{N}$ and $Z_{N}$; $H(Z_{N})$ denotes the entropy of $Z_{N}$; $H(p_{N}(Z|X),p_{A}(Z|X))$ denotes the cross-entropy between $p_{N}(Z|X)$ and $p_{A}(Z|X)$; ${KL}\left[\emph{p}_{N}(X)\ ||\ \emph{p}_{A}(X) \right]$ denotes the KL divergence between $p_{N}(X)$ and $p_{A}(X)$. Then:
\begin{equation}
\begin{split}\label{eq:LB_KL_reformulate2}
&\operatorname{KL}\left[p_{N}(X, Z)|| p_{A}(X, Z) \right] \\
=&~\emph{I}\left(X_{N},Z_{N}\right) - \emph{H}(Z_{N}) +\mathbb{E}_{p_{N}(X)}\left[ H(p_{N}(Z|X),p_{A}(Z|X))\right]
+KL\left[\emph{p}_{N}(X)||\emph{p}_{A}(X) \right].
\end{split}
\end{equation}
\end{proposition}

\begin{proof}
The KL divergence can be reformulated as:
\begin{equation}
\small
\begin{split}\nonumber 
\setlength{\abovedisplayskip}{3pt}
\setlength{\belowdisplayskip}{3pt}
&~\operatorname{KL}\left[p_{N}(X, Z)|| p_{A}(X, Z) \right] \\
=&~\mathbb{E}_{p_{N}(X, Z)}\left[ \log \frac{p_{N}(X, Z)}{p_{A}(X, Z)} \right]\\ 
=&~\mathbb{E}_{p_{N}(X, Z)}\left[ \log \frac{p_{N}(Z|X) \cdot p_{N}(X)}{p_{A}(Z|X) \cdot p_{A}(X)} \right]\\
=&~\mathbb{E}_{p_{N}(X, Z)}\left[ \log \frac{p_{N}(Z|X) \cdot p_{N}(X) \cdot p_{N}(Z)}{p_{A}(Z|X) \cdot p_{A}(X) \cdot p_{N}(Z)} \right]\\
=&~\mathbb{E}_{p_{N}(X, Z)} \left[\log\left( \frac{p_{N}(Z|X)}{p_{N}(Z)} \cdot {p_{N}(Z)} \cdot \frac{1} {p_{A}(Z|X)} \cdot \frac{p_{N}(X)}{p_{A}(X)} \right) \right].
\end{split} 
\end{equation}
The above formula is decomposed into four components. The first term refers to the mutual information between the original data $X_{N}$ and its latent representation $Z_{N}$:
\begin{equation}
\small
\begin{split}\nonumber
\setlength{\abovedisplayskip}{3pt}
\setlength{\belowdisplayskip}{3pt}
&\mathbb{E}_{p_{N}(X, Z)}\left[ \log \frac{p_{N}(Z|X)}{p_{N}(Z)}\right] \\
=&\mathbb{E}_{p_{N}(X, Z)}\left[ \log \frac{p_{N}(Z|X) \cdot p_{N}(X)}{p_{N}(X) \cdot p_{N}(Z)} \right] \\
=&\mathbb{E}_{p_{N}(X, Z)}\left[ \log \frac{p_{N}(X,Z)}{p_{N}(X) \cdot p_{N}(Z)} \right] \\ 
=&I\left(X_{N},Z_{N}\right).
\end{split}
\end{equation}
The second term refers to the negative entropy of $Z_{N}$:
\begin{equation}
\begin{split}\nonumber
&\mathbb{E}_{p_{N}(X, Z)}\left[ \log {p_{N}(Z)} \right] = -\mathbb{E}_{p_{N}(Z)}\left[ \log \frac{1}{p_{N}(Z)} \right]=-\emph{H}(Z_{N}).
\end{split}
\end{equation}
The third term refers to the expected value of the cross entropy between the conditional distributions $p_{A}(Z|X)$ and $p_{N}(Z|X)$:
\begin{equation}
\begin{split}\nonumber
&\mathbb{E}_{p_{N}(X, Z)}\left[\log \frac{1}{p_{A}(Z|X) }\right] \\
=& \mathbb{E}_{p_{N}(X)} \mathbb{E}_{p_{N}(Z|X)}\left[ - \log {p_{A}(Z|X) }\right] \\
=& \mathbb{E}_{p_{N}(X)}\left[ H(p_{N}(Z|X),p_{A}(Z|X))\right].
\end{split}
\end{equation}
The fourth term is a constant, since $p_{N}(X)$ and $p_{A}(X)$ are fixed when the dataset is given:
\begin{equation}
\begin{split}\nonumber%\label{eq:KL_X}
\mathbb{E}_{p_{N}(X, Z)}\left[ \log \frac{p_{N}(X)}{p_{A}(X)} \right] = \operatorname{KL}\left[\emph{p}_{N}(X)||\emph{p}_{A}(X) \right] = \emph{C}.
\end{split}
\end{equation}
Thus the KL divergence can be reformulated as:
\begin{equation}
\begin{split}\nonumber
&\operatorname{KL}\left[p_{N}(X, Z)|| p_{A}(X, Z) \right] \\
=&~\emph{I}\left(X_{N},Z_{N}\right) - \emph{H}(Z_{N}) +\mathbb{E}_{p_{N}(X)}\left[ H(p_{N}(Z|X),p_{A}(Z|X))\right]+KL\left[\emph{p}_{N}(X)||\emph{p}_{A}(X) \right].
\end{split}
\end{equation}

\end{proof}

\begin{proposition}
The third term in the objective Eq.~\eqref{eq:LB_KL_reformulate2}, i.e., $\mathbb{E}_{p_{N}(X)}\left[ H(p_{N}(Z|X),p_{A}(Z|X)\right]$, is non-negative.
\end{proposition}
\begin{proof}
We assume that $p_N(Z|X)$ and $p_A(Z|X)$ are separable at the latent space, \emph{i.e.,} for $X,Z \sim p_N(X,Z)$, the evaluated density $\log p_A(Z|X) \leqslant 0$.
This assumption is indeed consist with the fundamental assumption in~\cite{chandola2009anomaly}: data can be embedded into a certain representation space where normal instances and anomalies appear significantly different.
With the above assumption,\\ $\mathbb{E}_{p_N(X)}\left[ H(p_N(Z|X),p_A(Z|X)\right]$ is shown to be non-negative:
\begin{equation}
\begin{split}\nonumber
\setlength{\abovedisplayskip}{4pt}
\setlength{\belowdisplayskip}{4pt}
&\inf \mathbb{E}_{p_{N}(X)}\left[ H(p_{N}(Z|X),p_{A}(Z|X)\right]\\
=& \inf \mathbb{E}_{p_{N}(X,Z)} [ - \log p_{A}(Z|X)]  \\
\geqslant &\mathbb{E}_{p_{N}(X,Z)} [\inf \left(-\log p_{A}(Z|X)\right)] \\
\geqslant &0. 
\end{split}
\end{equation}
\end{proof}

\begin{proposition}
Assuming $Z$ follows an isotropic Gaussian, with mean $\boldsymbol{\mu}$, covariance $\Sigma$ and $Z \subseteq \mathbb{R}^{d}$, the entropy of $Z$, \emph{i.e.}, $\emph{H}(Z)$, is proportional to its log-variance for a fixed dimensionality $d$, without dependence on its mean $\boldsymbol{\mu}$.
\end{proposition}
\begin{proof}
For $Z$ with covariance $\Sigma$ and $Z \subseteq \mathbb{R}^{d}$,
\begin{equation}
\begin{split}\nonumber
    \emph{H}(Z)&=\mathbb{E}[-\log p(Z)]
    =-\int_{Z} p(Z) \log p(Z) \mathrm{d} Z
    \leq \frac{1}{2} \log ((2 \pi e)^{d} \operatorname{det} \Sigma),
\end{split}
\end{equation}
which holds with equality iff $Z$ is jointly Gaussian~\cite{cover2012elements}. Assuming $Z$ follows an isotropic Gaussian, $Z \sim N(\boldsymbol{\mu}, \sigma^{2} I)$ with $\sigma>0$, we get,
\begin{equation}
\begin{split}\nonumber
    \emph{H}(Z)=\frac{1}{2} \log ((2 \pi e)^{d} \operatorname{det} \sigma^{2} I)=\frac{d}{2}(1+\log (2 \pi \sigma^{2})) \propto \log \sigma^{2},
\end{split}
\end{equation}
which shows that the entropy of $Z$ is proportional to its log-variance for a fixed dimensionality $d$. The above proof has no dependence on the mean $\boldsymbol{\mu}$.
\end{proof}

\subsection{Analysis of Deep SAD}\label{sec32}
%\subsection{Deep SAD and Mathematical Analysis}\label{sec32}
Deep SAD builds upon Infomax principle, which maximizes the mutual information $I(X, Z)$ between data and latent representations with regularization on the representations. The objective function for Deep SAD is formulated as:
\begin{equation}\label{eq:sad}
\setlength{\abovedisplayskip}{2pt}
\setlength{\belowdisplayskip}{2pt}
    \max_\theta ~I(X, Z) + \beta (H(Z_A) - H(Z_N)),
\end{equation}
where regularization is enforced through entropy. For $\forall \mathbf{x}\in \mathcal{X}$, Deep SAD adopts an autoencoder consisting of an encoder $Enc(\cdot)$ and a decoder $Dec(\cdot)$: $\mathbf{z} =Enc(\mathbf{x}),~
\hat{\mathbf{x}}=Dec(\mathbf{z})$,
where $\hat{\mathbf{x}}$ is the reconstructed sample and $\mathbf{z}$ is the corresponding latent representation, and takes the following two-step process to implement the above objective function.

\noindent\textbf{(i) Autoencoder Pre-training:} To \emph{maximize the mutual information} between the data and the latent representations, 
a reconstruction loss is adopted to pre-train the autoencoder:
\begin{equation}
\setlength{\abovedisplayskip}{2pt}
\setlength{\belowdisplayskip}{2pt}
    \mathcal{L}_{rec} = \frac{1}{n+m}\sum_{i = 1}^{n+m}\|\hat{\mathbf{x}}_{i}-\mathbf{x}_{i}\|^2,
\end{equation}
where $\mathbf{x}_{1}, \cdots, \mathbf{x}_{n+m} \in \mathcal{X}$. 

\noindent\textbf{(ii) Encoder Fine-turning:} 
To further regularize the entropy of the latent representations, the encoder is fine-turned with an SVDD loss, 
\begin{equation}
\setlength{\abovedisplayskip}{2pt}
\setlength{\belowdisplayskip}{2pt}
    \mathcal{L}_{SVDD} = \frac{1}{n}\sum_{i = 1}^{n}\|\mathbf{z}^u_i-\mathbf{c}\|_2 + \frac{\eta}{m}\sum_{j = 1}^{m}\|\mathbf{z}^l_j-\mathbf{c}\|_2^{y_j},
\end{equation}
where $\mathbf{z}^u_1, \cdots, \mathbf{z}^u_n \in \mathcal{Z}$ are the corresponding latent representations of unlabeled samples $\mathbf{x}^u_1, \cdots, \mathbf{x}^u_n$, $\mathbf{z}^l_1, \cdots, \mathbf{z}^l_m \in \mathcal{Z}$ are the corresponding latent representations of labeled samples $\mathbf{x}^l_1, \cdots, \mathbf{x}^l_m$, and $\eta$ is set as 1. The hypersphere center $\mathbf{c}$ is set as the mean of the outputs obtained from a forward pass of the encoder for all the data. In fact, Deep SAD does not use the coefficient $\beta$ in Eq.~\eqref{eq:sad}, because the two terms are separately optimized in two stages.

We now argue that the reason for the two-stage implementation of Deep SAD is the contradiction between the optimization of mutual information and entropy. For example, when the latent representations have extremely low entropy, especially zero in the extreme case, the model can be considered as mapping all data into a constant in which the mutual information is restricted to zero, which contradicts with the mutual information maximization. The two-stage implementation for Deep SAD avoids directly facing the above contradiction. 

\renewcommand \arraystretch{0.8}
\begin{table}[h]
  \centering
  \scriptsize
  \caption{Model Structure of ESAD.}
  \label{tal:esad}
  \setlength{\tabcolsep}{1pt}{
\begin{tabular}{C{1.4cm}C{2.4cm}C{4.0cm}}
      \toprule
      Layer & Input & Output\\
      \cmidrule(lr){1-1}\cmidrule(lr){2-2}\cmidrule(lr){3-3}
      $3\times 3\times 64$ & $x\ (3 \times H \times W)$ & $x_{0-1}\ (64 \times H \times W)$\\
      $3\times 3\times 64$ & $x_{0-1}$ & $x_{0-2}\ (64 \times H \times W)$\\ 
      MaxPool & $x_{0-2}$ & $x_{1-1}\ (64 \times 1/2 H \times 1/2W)$ \\
      $3\times 3\times 128$ & $x_{1-1}$ & $x_{1-2}\ (128 \times 1/2 H \times 1/2W)$\\ 
      $3\times 3\times 128$ & $x_{1-2}$ & $x_{1-3}\ (128 \times 1/2 H \times 1/2W)$\\ 
      MaxPool & $x_{1-3}$ & $x_{2-1}\ (128 \times 1/4 H \times 1/4W)$ \\
      $3\times 3\times 256$ & $x_{2-1}$ & $x_{2-2}\ (256 \times 1/4 H \times 1/4W)$\\ 
      $3\times 3\times 256$ & $x_{2-2}$ & $x_{2-3}\ (256 \times 1/4 H \times 1/4W)$\\ 
      MaxPool & $x_{2-3}$ & $x_{3-1}\ (256 \times 1/8 H \times 1/8W)$ \\
      $3\times 3\times 512$ & $x_{3-1}$ & $x_{3-2}\ (256 \times 1/8 H \times 1/8W)$\\ 
      $3\times 3\times 512$ & $x_{3-2}$ & $x_{3-3}\ (256 \times 1/8 H \times 1/8W)$\\ 
      MaxPool & $x_{3-3}$ & $x_{4-1}\ (256 \times 1/8 H \times 1/16W)$ \\
      $3\times 3\times 512$ & $x_{4-1}$ & $x_{4-2}\ (512 \times 1/16 H \times 1/16W)$\\ 
      $3\times 3\times 512$ & $x_{4-2}$ & $z\ (512 \times 1/16 H \times 1/16W)$\\ 
      \cmidrule(lr){1-1}\cmidrule(lr){2-2}\cmidrule(lr){3-3}
      UpSample & $z$ & $up_{3-1}\ (512 \times 1/8 H \times 1/8W)$ \\
      $3\times 3\times 256$ & $up_{3-1}$ & $up_{3-2}\ (256 \times 1/8 H \times 1/8W)$\\ 
      $3\times 3\times 256$ & $up_{3-2}$ & $up_{3-3}\ (256 \times 1/8 H \times 1/8W)$\\ 
      UpSample & $up_{3-3}$ & $up_{2-1}\ (256 \times 1/4 H \times 1/4W)$ \\
      $3\times 3\times 128$ & $up_{2-1}$ & $up_{2-2}\ (128 \times 1/4 H \times 1/4W)$\\ 
      $3\times 3\times 128$ & $up_{2-2}$ & $up_{2-3}\ (128 \times 1/4 H \times 1/4W)$\\ 
      UpSample & $up_{2-3}$ & $up_{1-1}\ (128 \times 1/2 H \times 1/2W)$ \\
      $3\times 3\times 64$ & $up_{1-1}$ & $up_{1-2}\ (64 \times 1/2 H \times 1/2W)$\\ 
      $3\times 3\times 64$ & $up_{1-2}$ & $up_{1-3}\ (64 \times 1/2 H \times 1/2W)$\\ 
      UpSample & $up_{1-3}$ & $up_{0-1}\ (64 \times H \times W)$ \\
      $3\times 3\times 64$ & $up_{0-1}$ & $up_{0-2}\ (64 \times H \times W)$\\ 
      $3\times 3\times 64$ & $up_{0-2}$ & $up_{0-3}\ (64 \times H \times W)$\\ 

      $3\times 3\times 3$ & $up_{0-3}$ & $\hat{x}\ (3 \times H \times W)$\\ 
      \cmidrule(lr){1-1}\cmidrule(lr){2-2}\cmidrule(lr){3-3}
      $3\times 3\times 64$ & $\hat{x}$ & $x_{5-1}\ (64 \times H \times W)$\\
      $3\times 3\times 64$ & $x_{5-1}$ & $x_{5-2}\ (64 \times H \times W)$\\ 
      MaxPool & $x_{5-2}$ & $x_{6-1}\ (64 \times 1/2 H \times 1/2W)$ \\
      $3\times 3\times 128$ & $x_{6-1}$ & $x_{6-2}\ (128 \times 1/2 H \times 1/2W)$\\ 
      $3\times 3\times 128$ & $x_{6-2}$ & $x_{6-3}\ (128 \times 1/2 H \times 1/2W)$\\ 
      MaxPool & $x_{6-3}$ & $x_{7-1}\ (128 \times 1/4 H \times 1/4W)$ \\
      $3\times 3\times 256$ & $x_{7-1}$ & $x_{7-2}\ (256 \times 1/4 H \times 1/4W)$\\ 
      $3\times 3\times 256$ & $x_{7-2}$ & $x_{7-3}\ (256 \times 1/4 H \times 1/4W)$\\ 
      MaxPool & $x_{7-3}$ & $x_{8-1}\ (256 \times 1/8 H \times 1/8W)$ \\
      $3\times 3\times 512$ & $x_{8-1}$ & $x_{8-2}\ (256 \times 1/8 H \times 1/8W)$\\ 
      $3\times 3\times 512$ & $x_{8-2}$ & $x_{8-3}\ (256 \times 1/8 H \times 1/8W)$\\ 
      MaxPool & $x_{8-3}$ & $x_{9-1}\ (256 \times 1/8 H \times 1/16W)$ \\
      $3\times 3\times 512$ & $x_{9-1}$ & $x_{9-2}\ (512 \times 1/16 H \times 1/16W)$\\ 
      $3\times 3\times 512$ & $x_{9-2}$ & $\hat{z}\ (512 \times 1/16 H \times 1/16W)$\\ 
      \bottomrule
    \end{tabular}}
\end{table}

\renewcommand \arraystretch{0.85}
\begin{table}[t]
\centering
\caption{Classic anomaly detection benchmarks~\cite{Rayana2016}.}
\label{tal:dataset}
\small
\setlength{\tabcolsep}{1.3pt}{
\begin{tabular}{C{2.0cm}C{1.8cm}C{2.0cm}C{2.0cm}}
\toprule
Dataset & Numbers & Dimensions & \#outliers (\%) \\
\cmidrule(lr){1-1} \cmidrule(lr){2-2} \cmidrule(lr){3-3} \cmidrule(lr){4-4}
arrhythmia & 452 & 274 & 66 (14.6\%) \\
cardio & 1,831 & 21 & 176 (9.6\%) \\
satellite & 6,435 & 36 & 2,036 (31.6\%) \\
satimage-2 & 5,803 & 36 & 71 (1.2\%) \\
shuttle & 49,097 & 9 & 3,511 (7.2\%) \\
thyroid & 3,772 & 6 & 93 (2.5\%) \\
\bottomrule
\end{tabular}}
\end{table}

\subsection{Model Architecture and Training Details} 
The model architecture for ESAD is shown in Table~\ref{tal:esad}. For the training, we use stochastic gradient descent (SGD)~\cite{bottou2010large} optimizer with default hyperparameters in Pytorch. ESAD is trained using a batch size of 32 for $200$ epochs with NVIDIA GTX 2080Ti. The learning rate is initially set as 0.1, and is divided by 2 every $50$ epoch.

\subsection{Datasets}
\noindent\textbf{Natural Image Datasets.} MNIST~\cite{lecun1998mnist}, a dataset consists of 70,000 $28\times28$ handwritten grayscale digit images; Fashion-MNIST~\cite{xiao2017fashion}, a relatively new dataset comprising $28\times28$ grayscale images of 70,000 fashion products from 10 categories, with 7,000 images per category; CIFAR-10~\cite{krizhevsky2009learning}, a dataset consists of 60,000 $32\times32$ RGB images of 10 classes, with 6,000 images for per class.

\noindent\textbf{Medical Image Datasets.} Following \cite{tang2019deep,tuluptceva2020anomaly}, we examine the detection of metastases in H\&E stained images of lymph nodes in Camelyon16~\cite{bejnordi2017diagnostic} and the recognition of fourteen diseases on the chest X-rays in the NIH dataset~\cite{wang2017chestx}. 

For the NIH dataset, images without any disease marker were considered normal. Pulmonary and cardiac abnormalities in this dataset include atelectasis, effusion, infiltration, mass, nodule, pneumonia, pneumothorax, consolidation, edema, emphysema, fibrosis, pleural thickening, hernia and cardiomegaly, which are all considered anomalous. Following~\cite{tang2019deep,tuluptceva2020anomaly}, we split the dataset into two sub-datasets having only posteroanterior (PA) or anteroposterior (AP) projections. Note that in the training set, the ratios of labeled anomalous samples are 3.9\% for AP and 3.3\% for PA. We also experiment on a subset containing clearer normal/anomalous cases \cite{tang2019deep}. This subset consists of 5110 normal and 857 anomalous images for training, and 677 normal and 677 anomalous images for testing. 

For the Camelyon16 dataset, we sample the Vahadane-normalized~\cite{vahadane2016structure} $64\times 64$ tiles from the fully normal slides with magnification of $10\times$, and treat these as normal. Tiles with metastases are treated as anomalous. It contains 7612 normal and 200 anomalous training images, and 4000 (normal) + 817 (anomalous) images for the test.

\noindent\textbf{Classic anomaly detection benchmark datasets.} We use six non-image classic anomaly detection benchmark datasets~\cite{Rayana2016}. Following~\cite{SAD}, for the evaluation, we consider random train-to-test set splits of 60:40 while maintaining the original proportion of anomalies in each set. The supplementary details of the classic anomaly detection benchmarks~\cite{Rayana2016} are shown in Table~\ref{tal:dataset}.

\subsection{Competing Methods}
We consider several shallow unsupervised methods, deep unsupervised anomaly detection competitors and semi-supervised anomaly detection approaches as baselines. Complete details are shown as follows:

\noindent \textbf{(1) OC-SVM/SVDD}~\cite{scholkopf2001estimating,tax2004support}: The OC-SVM and SVDD are equivalent for the Gaussian/RBF kernel. OC-SVM/SVDD here have unfair advantages by selecting their hyperparameters to maximize AUC on a subset $(10\%)$ of the test set to establish a strong baseline. The RBF scale parameters $\gamma \in\left\{2^{-7}, 2^{-6}, \ldots 2^{2}\right\}$ are considered and the best performing one is selected. Then the best final results are reported over $\nu$-parameter, where $\nu \in$ \{0.01,0.05,0.1,0.2,0.5\}.

\noindent \textbf{(2) Isolation Forest}~\cite{liu2008isolation}: The number of trees is set as $t=100$ and the sub-sampling size is set as $\psi=256$ as recommended in the original work.

\noindent \textbf{(3) SSAD}~\cite{gornitz2013toward}: SSAD also have the unfair advantages the same as OC-SVM/SVDD. The scale parameters $\gamma$ of the RBF kernel are selected from $\gamma \in\left\{2^{-7}, 2^{-6}, \ldots 2^{2}\right\}$ and then report the best performing one. Otherwise we set the hyperparameters as recommend by the original authors to $\kappa=1, \kappa=1, \eta_{u}=1,$ and $\eta_{l}=1$~\cite{gornitz2013toward}.

\noindent \textbf{(4) Convolutional Autoencoder (CAE)}~\cite{masci2011stacked}: The autoencoders are trained on the MSE reconstruction loss that also serves as the anomaly score.

\noindent \textbf{(5) Deep SVDD}~\cite{ruff2018deep}: Both variants, Soft-Boundary Deep SVDD and One-Class Deep SVDD are considered as unsupervised baselines and always report the better performance as the unsupervised result. For Soft-Boundary Deep SVDD, The radius $R$ on every mini-batch is optimally solved. For Deep SVDD, all the bias terms from a network are removed to prevent a hypersphere collapse as recommended by the authors in the original work~\cite{ruff2018deep}.

\noindent \textbf{(6) SS-DGM}~\cite{kingma2014semi}: We consider both the $M2$ and $M1+M2$ model and always report the better performing result. Other settings are following the original work~\cite{kingma2014semi}.

\noindent \textbf{(7) Deep SAD~\cite{SAD}}: The results are borrow from~\cite{SAD}. We set $\lambda=10^{-6}$ and equally weight the unlabeled and labeled examples by setting $\eta=1$ if not reported otherwise.

\renewcommand \arraystretch{0.9}
\begin{table}[t]
\centering
\caption{Average area under the ROC curve (AUC) in \% on natural image datasets, comparing with unsupervised anomaly detection methods. ``$\dag$'' denotes the highest test AUC among multiple running for the strong baselines. ``$*$'' denotes the highest test AUC among all training epochs for the stronger baselines. We report the results of unsupervised ESAD, where we ignore the labeled data in the training set. Emphasizing that ESAD focuses on the semi-supervised setting but not the unsupervised setting.}
\label{tal:un}
\footnotesize
\begin{tabular}{C{2.5cm}C{1.3cm}C{1.3cm}C{1.5cm}}
\toprule
Method & MNIST & F-MNIST & CIFAR-10\\
\cmidrule(lr){1-1} \cmidrule(lr){2-2} \cmidrule(lr){3-3} \cmidrule(lr){4-4}
CAE~\cite{masci2011stacked} & 92.9 $\pm$ 5.7 & 90.2 $\pm$ 5.8 & 56.2 $\pm$ 13.2\\
IF Hybrid~\cite{liu2008isolation} & 90.5 $\pm$ 5.3 & 82.5 $\pm$ 8.1 & 59.9 $\pm$ 6.7 \\
Deep SVDD~\cite{ruff2018deep} & 92.8 $\pm$ 4.9 & 89.2 $\pm$ 6.2 & 60.9 $\pm$ 9.4\\ 
AnoGAN$^\dag$~\cite{schlegl2017unsupervised} & 93.7 & - & 61.2\\
ALOCC$^*$~\cite{Sabokrou2018Adversarially} & 93.3 & - & 62.2\\
ADGAN$^*$~\cite{deecke2018anomaly} & 94.7 & 88.4 & 62.4\\
OC-SVM Hybrid~\cite{scholkopf2001estimating} & 96.3 $\pm$ 2.5 & 91.2 $\pm$ 4.7 & 63.8 $\pm$ 9.0\\
OCGAN$^\dag$~\cite{perera2019ocgan} & 97.5 & - & 65.6\\
GANomaly$^*$~\cite{akccay2019skip} & 92.8 & 80.9 & 69.5\\
P-KDGAN$^\dag$~\cite{zhang2020p} & 97.8 & - & 73.8\\
DGEO$^\dag$~\cite{golan2018deep} & 98.0 & 93.5 & 86.0\\
\cmidrule(lr){1-1} \cmidrule(lr){2-2} \cmidrule(lr){3-3} \cmidrule(lr){4-4}
ESAD (unsupervised) & 98.5 $\pm$ 1.3 & 94.0 $\pm$ 4.5 & 78.8 $\pm$ 6.5\\
ESAD & 99.6 $\pm$ 0.3 & 95.9 $\pm$ 4.0 & 88.5 $\pm$ 6.9\\
\bottomrule
\end{tabular}
\vspace{-10pt}
\end{table}

\renewcommand \arraystretch{1.0}
\begin{table*}[t]
\centering
\caption{Complete results of \textbf{experimental scenario (ii)}, where we pollute the unlabeled part of the training set with (unknown) anomalies. We report the avg. AUC in \% with st. dev. computed over 90 experiments at various ratios $\gamma_{p}$.}
\label{tal:2}
\tiny
\setlength{\tabcolsep}{1.0pt}{
\begin{tabular}{C{1.2cm}C{0.5cm}C{1.0cm}C{1.0cm}C{1.0cm}C{1.0cm}C{1.0cm}C{1.0cm}C{1.0cm}C{0.6cm}C{1.0cm}C{1.0cm}}
\toprule
Data & $\gamma_p$ & \makecell[c]{OC-SVM \\Hybrid~\cite{scholkopf2001estimating}} & \makecell[c]{IF \\Hybrid~\cite{liu2008isolation}} & \makecell[c]{CAE\\\cite{masci2011stacked}} & \makecell[c]{Deep\\ SVDD~\cite{ruff2018deep}} & \makecell[c]{SSAD\\ Hybrid~\cite{gornitz2013toward}} & \makecell[c]{SS-DGM\\\cite{kingma2014semi}} & \makecell[c]{Deep\\ SAD~\cite{SAD}} &
\makecell[c]{TLSAD\\~\cite{TLSAD}} & \makecell[c]{ESAD \\ (ours)} & \makecell[c]{Supervised\\ Classifier}\\
\cmidrule(lr){1-1} \cmidrule(lr){2-2} \cmidrule(lr){3-3} \cmidrule(lr){4-4} \cmidrule(lr){5-5} \cmidrule(lr){6-6} \cmidrule(lr){7-7} \cmidrule(lr){8-8} \cmidrule(lr){9-9} \cmidrule(lr){10-10} \cmidrule(lr){11-11} \cmidrule(lr){12-12}
& .00 & 96.3 $\pm$ 2.5 & 90.5 $\pm$ 5.3 & 92.9 $\pm$ 5.7 & 92.8 $\pm$ 4.9 & 97.4 $\pm$ 2.0 & 92.2 $\pm$ 5.6 & 96.7 $\pm$ 2.4 & 96.9 & \textbf{99.4} $\pm$ \textbf{0.3} & 94.5 $\pm$ 4.6 \\
& .01 & 95.6 $\pm$ 2.5 & 90.6 $\pm$ 5.0 & 91.3 $\pm$ 6.1 & 92.1 $\pm$ 5.1 & 95.2 $\pm$ 2.3 & 92.0 $\pm$ 6.0 & 95.5 $\pm$ 3.3 & 94.5 & \textbf{99.2} $\pm$ \textbf{0.6} & 91.5 $\pm$ 5.9 \\
MNIST & .05 & 93.8 $\pm$ 3.9 & 89.7 $\pm$ 6.0 & 87.2 $\pm$ 7.1 & 89.4 $\pm$ 5.8 & 89.5 $\pm$ 3.9 & 91.0 $\pm$ 6.9 & 93.5 $\pm$ 4.1 & 94.0 & \textbf{98.5} $\pm$ \textbf{1.0} & 86.7 $\pm$ 7.4 \\
& .10 & 91.4 $\pm$ 5.1 & 88.2 $\pm$ 6.5 & 83.7 $\pm$ 8.4 & 86.5 $\pm$ 6.8 & 86.0 $\pm$ 4.6 & 89.7 $\pm$ 7.5 & 91.2 $\pm$ 4.9 & 93.5 & \textbf{97.8} $\pm$ \textbf{1.3} & 83.6 $\pm$ 8.2 \\
& .20 & 85.9 $\pm$ 7.6 & 85.3 $\pm$ 7.9 & 78.6 $\pm$ 10.3 & 81.5 $\pm$ 8.4 & 82.1 $\pm$ 5.4 & 87.4 $\pm$ 8.6 & 86.6 $\pm$ 6.6 & 88.6 & \textbf{96.7} $\pm$ \textbf{2.0} & 79.7 $\pm$ 9.4 \\
\cmidrule(lr){1-1} \cmidrule(lr){2-2} \cmidrule(lr){3-3} \cmidrule(lr){4-4} \cmidrule(lr){5-5} \cmidrule(lr){6-6} \cmidrule(lr){7-7} \cmidrule(lr){8-8} \cmidrule(lr){9-9} 
\cmidrule(lr){10-10}
\cmidrule(lr){11-11} \cmidrule(lr){12-12}
& .00 & 91.2 $\pm$ 4.7 & 82.5 $\pm$ 8.1 & 90.2 $\pm$ 5.8 & 89.2 $\pm$ 6.2 & 90.5 $\pm$ 5.9 & 71.4 $\pm$ 12.7 & 90.5 $\pm$ 6.5 & 91.4 & \textbf{95.6} $\pm$ \textbf{4.1} & 76.8 $\pm$ 13.2 \\
& .01 & 91.5 $\pm$ 4.6 & 84.9 $\pm$ 7.2 & 87.1 $\pm$ 7.3 & 86.3 $\pm$ 6.3 & 87.8 $\pm$ 6.1 & 71.2 $\pm$ 14.3 & 87.2 $\pm$ 7.1 & 92.3 & \textbf{95.5} $\pm$ \textbf{4.1} & 67.3 $\pm$ 8.1 \\
F-MNIST & .05 & 90.7 $\pm$ 4.9 & 85.5 $\pm$ 7.2 & 81.6 $\pm$ 9.6 & 80.6 $\pm$ 7.1 & 82.7 $\pm$ 7.8 & 71.9 $\pm$ 14.3 & 81.5 $\pm$ 8.5& 89.8  & \textbf{94.5} $\pm$ \textbf{4.5} & 59.8 $\pm$ 4.6 \\
& .10 & 89.3 $\pm$ 6.2 & 85.5 $\pm$ 7.7 & 77.4 $\pm$ 11.1 & 76.2 $\pm$ 7.3 & 79.8 $\pm$ 9.0 & 72.5 $\pm$ 15.5 & 78.2 $\pm$ 9.1 & 90.2  & \textbf{93.6} $\pm$ \textbf{4.7} & 56.7 $\pm$ 4.1 \\
& .20 & 88.1 $\pm$ 6.9 & 86.3 $\pm$ 7.4 & 72.5 $\pm$ 12.6 & 69.3 $\pm$ 6.3 & 74.3 $\pm$ 10.6 & 70.8 $\pm$ 16.0 & 74.8 $\pm$ 9.4 & 88.4 & \textbf{92.5} $\pm$ \textbf{4.9} & 53.9 $\pm$ 2.9 \\
\cmidrule(lr){1-1} \cmidrule(lr){2-2} \cmidrule(lr){3-3} \cmidrule(lr){4-4} \cmidrule(lr){5-5} \cmidrule(lr){6-6} \cmidrule(lr){7-7} \cmidrule(lr){8-8} \cmidrule(lr){9-9} 
\cmidrule(lr){10-10}
\cmidrule(lr){11-11} \cmidrule(lr){12-12}
& .00 & 63.8 $\pm$ 9.0 & 59.9 $\pm$ 6.7 & 56.2 $\pm$ 13.2 & 60.9 $\pm$ 9.4 & 73.3 $\pm$ 8.4 & 50.8 $\pm$ 4.7 & 77.9 $\pm$ 7.2 & 80.0  & \textbf{86.9} $\pm$ \textbf{6.8} & 63.5 $\pm$ 8.0 \\
& .01 & 63.8 $\pm$ 9.3 & 59.9 $\pm$ 6.7 & 56.2 $\pm$ 13.1 & 60.5 $\pm$ 9.4 & 72.8 $\pm$ 8.1 & 51.1 $\pm$ 4.7 & 76.5 $\pm$ 7.2 & 76.4  & \textbf{86.5} $\pm$ \textbf{6.9} & 62.9 $\pm$ 7.3 \\
CIFAR-10 & .05 & 62.6 $\pm$ 9.2 & 59.6 $\pm$ 6.4 & 55.7 $\pm$ 13.3 & 59.6 $\pm$ 9.8 & 71.0 $\pm$ 8.4 & 50.1 $\pm$ 2.9 & 74.0 $\pm$ 6.9 & 75.9 & \textbf{84.3} $\pm$ \textbf{7.4} & 62.2 $\pm$ 8.2 \\
& .10 & 62.9 $\pm$ 8.2 & 59.1 $\pm$ 6.6 & 55.4 $\pm$ 13.3 & 58.6 $\pm$ 10.0 & 69.3 $\pm$ 8.5 & 50.5 $\pm$ 3.6 & 71.8 $\pm$ 7.0 & 72.6 & \textbf{81.9} $\pm$ \textbf{7.7} & 60.6 $\pm$ 8.3 \\
& .20 & 61.9 $\pm$ 8.1 & 58.3 $\pm$ 6.2 & 54.6 $\pm$ 13.3 & 57.0 $\pm$ 10.6 & 67.9 $\pm$ 8.1 & 50.1 $\pm$ 1.7 & 68.5 $\pm$ 7.1 & 71.4  & \textbf{79.8} $\pm$ \textbf{8.8} & 58.5 $\pm$ 6.7 \\
\bottomrule
\end{tabular}}
\end{table*}

\renewcommand \arraystretch{1.0}
\begin{table*}[!htb]
\centering
\caption{Complete results of \textbf{experimental scenario (iii)}, where we increase the number of anomaly classes $k_{l}$ included in the labeled training data. We report the avg. AUC in \% with st. dev. computed over 100 experiments at various numbers $k_{l}$. Note that unsupervised methods~\cite{scholkopf2001estimating,liu2008isolation,masci2011stacked,ruff2018deep} cannot be applied to the semi-supervised setting when $k_l \not= 0$, while SS-DGM~\cite{kingma2014semi} and the supervised classifier are not compatible for the unsupervised setting when $k_l = 0$.}
\label{tal:3}
\tiny
\setlength{\tabcolsep}{1.0pt}{
\begin{tabular}{C{1.2cm}C{0.5cm}C{1.0cm}C{1.0cm}C{1.0cm}C{1.0cm}C{1.0cm}C{1.0cm}C{1.0cm}C{1.0cm}C{1.0cm}C{1.0cm}}
\toprule
Data & $k_l$ & \makecell[c]{OC-SVM \\Hybrid~\cite{scholkopf2001estimating}} & \makecell[c]{IF \\Hybrid~\cite{liu2008isolation}} & \makecell[c]{CAE\\\cite{masci2011stacked}} & \makecell[c]{Deep\\ SVDD~\cite{ruff2018deep}} & \makecell[c]{SSAD\\ Hybrid~\cite{gornitz2013toward}} & \makecell[c]{SS-DGM\\\cite{kingma2014semi}} & \makecell[c]{Deep\\ SAD~\cite{SAD}} & \makecell[c]{ESAD \\ (ours)} & \makecell[c]{Supervised\\ Classifier}\\
\cmidrule(lr){1-1} \cmidrule(lr){2-2} \cmidrule(lr){3-3} \cmidrule(lr){4-4} \cmidrule(lr){5-5} \cmidrule(lr){6-6} \cmidrule(lr){7-7} \cmidrule(lr){8-8} \cmidrule(lr){9-9} 
\cmidrule(lr){10-10} 
\cmidrule(lr){11-11} 
 & 0 & 91.4 $\pm$ 5.1 & 88.2 $\pm$ 6.5 & 83.7 $\pm$ 8.4 & 86.5 $\pm$ 6.8 & 91.4 $\pm$ 5.1 & & 86.5 $\pm$ 6.8 & \textbf{92.7} $\pm$ \textbf{3.8} & \\
 & 1 & & & & & 86.0 $\pm$ 4.6 & 89.7 $\pm$ 7.5 & 91.2 $\pm$ 4.9 & \textbf{97.8} $\pm$ \textbf{1.3} & 83.6 $\pm$ 8.2\\
MNIST  & 2 & & & & & 87.7 $\pm$ 3.8 & 92.8 $\pm$ 5.3 & 92.0 $\pm$ 3.6 & \textbf{98.2} $\pm$ \textbf{0.9} & 90.3 $\pm$ 4.6 \\
 & 3 & & & & & 89.8 $\pm$ 3.3 & 94.9 $\pm$ 4.2 & 94.7 $\pm$ 2.8 & \textbf{99.1} $\pm$ \textbf{0.6} & 93.9 $\pm$ 2.8 \\
 & 5 & & & & & 91.9 $\pm$ 3.0 & 96.7 $\pm$ 2.3 & 97.3 $\pm$ 1.8 & \textbf{99.3} $\pm$ \textbf{0.5} & 96.9 $\pm$ 1.7 \\
\cmidrule(lr){1-1} \cmidrule(lr){2-2} \cmidrule(lr){3-3} \cmidrule(lr){4-4} \cmidrule(lr){5-5} \cmidrule(lr){6-6} \cmidrule(lr){7-7} \cmidrule(lr){8-8} \cmidrule(lr){9-9} 
\cmidrule(lr){10-10}
\cmidrule(lr){11-11} 
 & 0 & 89.3 $\pm$ 6.2 & 85.5 $\pm$ 7.7 & 77.4 $\pm$ 11.1 & 76.2 $\pm$ 7.3 & 89.3 $\pm$ 6.2 & & 76.2 $\pm$ 7.3 & \textbf{91.2} $\pm$ \textbf{5.4} &\\
 & 1 & & & & & 79.8 $\pm$ 9.0 & 72.5 $\pm$ 15.5 & 78.2 $\pm$ 9.1 & \textbf{93.6} $\pm$ \textbf{4.7} & 56.7 $\pm$ 4.1 \\
F-MNIST  & 2 & & & & & 80.1 $\pm$ 10.5 & 74.3 $\pm$ 15.4 & 80.5 $\pm$ 8.2 & \textbf{94.7} $\pm$ \textbf{4.6} & 62.3 $\pm$ 2.9 \\
 & 3 & & & & & 83.8 $\pm$ 9.4 & 77.5 $\pm$ 14.7 & 83.9 $\pm$ 7.4 & \textbf{95.8} $\pm$ \textbf{4.8} & 67.3 $\pm$ 3.0 \\
 & 5 & & & & & 86.8 $\pm$ 7.7 & 79.9 $\pm$ 13.8 & 87.3 $\pm$ 6.4 & \textbf{96.7} $\pm$ \textbf{4.3} & 75.3 $\pm$ 2.7 \\
\cmidrule(lr){1-1} \cmidrule(lr){2-2} \cmidrule(lr){3-3} \cmidrule(lr){4-4} \cmidrule(lr){5-5} \cmidrule(lr){6-6} \cmidrule(lr){7-7} \cmidrule(lr){8-8} \cmidrule(lr){9-9} 
\cmidrule(lr){10-10}
\cmidrule(lr){11-11} 
 & 0 & 62.9 $\pm$ 8.2 & 59.1 $\pm$ 6.6 & 55.4 $\pm$ 13.3 & 86.6 $\pm$ 10.0 & 62.9 $\pm$ 8.2 & & 58.6 $\pm$ 10.0 & \textbf{73.5} $\pm$ \textbf{6.8} & \\
 & 1 & & & & & 69.3 $\pm$ 8.5 & 50.5 $\pm$ 3.6 & 71.8 $\pm$ 7.0 & \textbf{81.9} $\pm$ \textbf{7.7} & 60.6 $\pm$ 8.3 \\
CIFAR-10 & 2 & & & & & 72.3 $\pm$ 7.5 & 50.3 $\pm$ 2.4 & 75.2 $\pm$ 6.4 & \textbf{83.8} $\pm$ \textbf{6.0} & 61.0 $\pm$ 6.6 \\
 & 3 & & & &  & 73.3 $\pm$ 7.0 & 50.0 $\pm$ 0.7 & 77.5 $\pm$ 5.9 & \textbf{84.9} $\pm$ \textbf{8.1} & 62.7 $\pm$ 6.8 \\
 & 5 & & & & & 74.2 $\pm$ 6.5 & 50.0 $\pm$ 1.0 & 80.4 $\pm$ 4.6 & \textbf{86.7} $\pm$ \textbf{7.0} & 60.9 $\pm$ 4.6 \\
\bottomrule
\end{tabular}}
\end{table*}

\begin{figure*}[t]
\centering
\includegraphics[width=12cm]{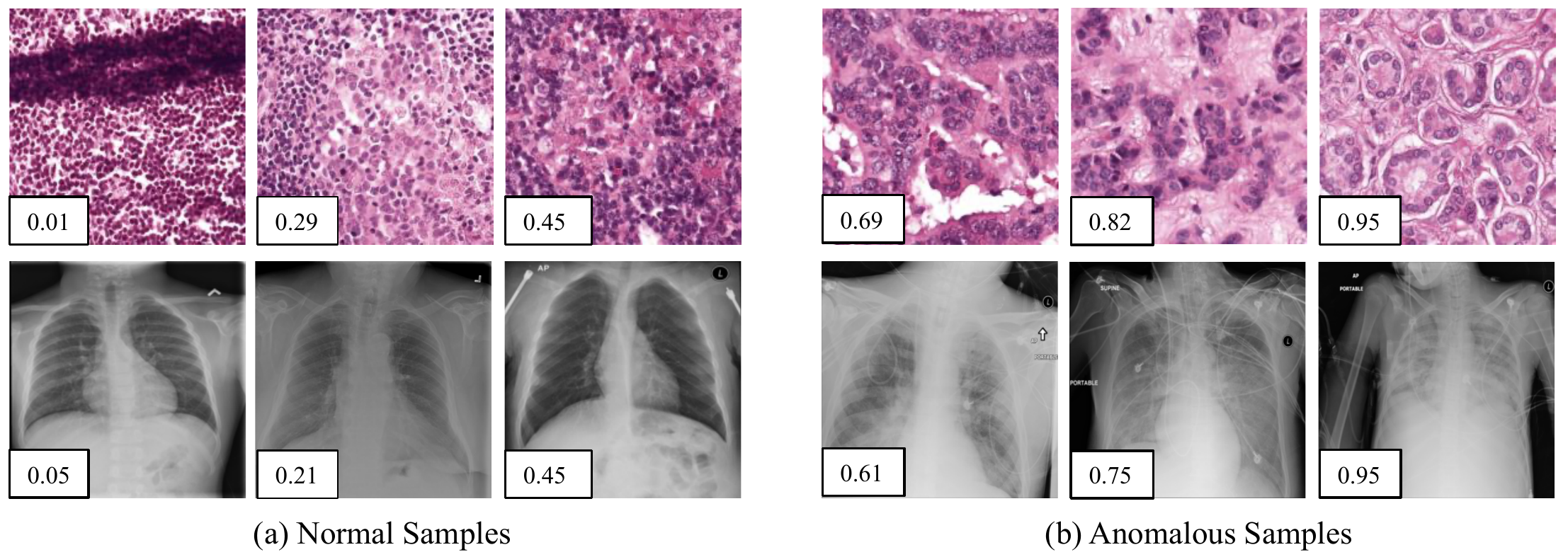}
\caption{Examples of normal (left) and anomalous (right) samples of H\&E-stained lymph node of Camelyon16 challenge \cite{bejnordi2017diagnostic} (top) and chest X-rays of NIH dataset \cite{wang2017chestx} (bottom). We show the predicted anomaly score by the proposed method. The higher the score, the more likely to be an anomaly. Best viewed in color.}
\label{img:vis_medical}
\end{figure*}

\begin{figure}[t]
\centering
\includegraphics[width=5cm]{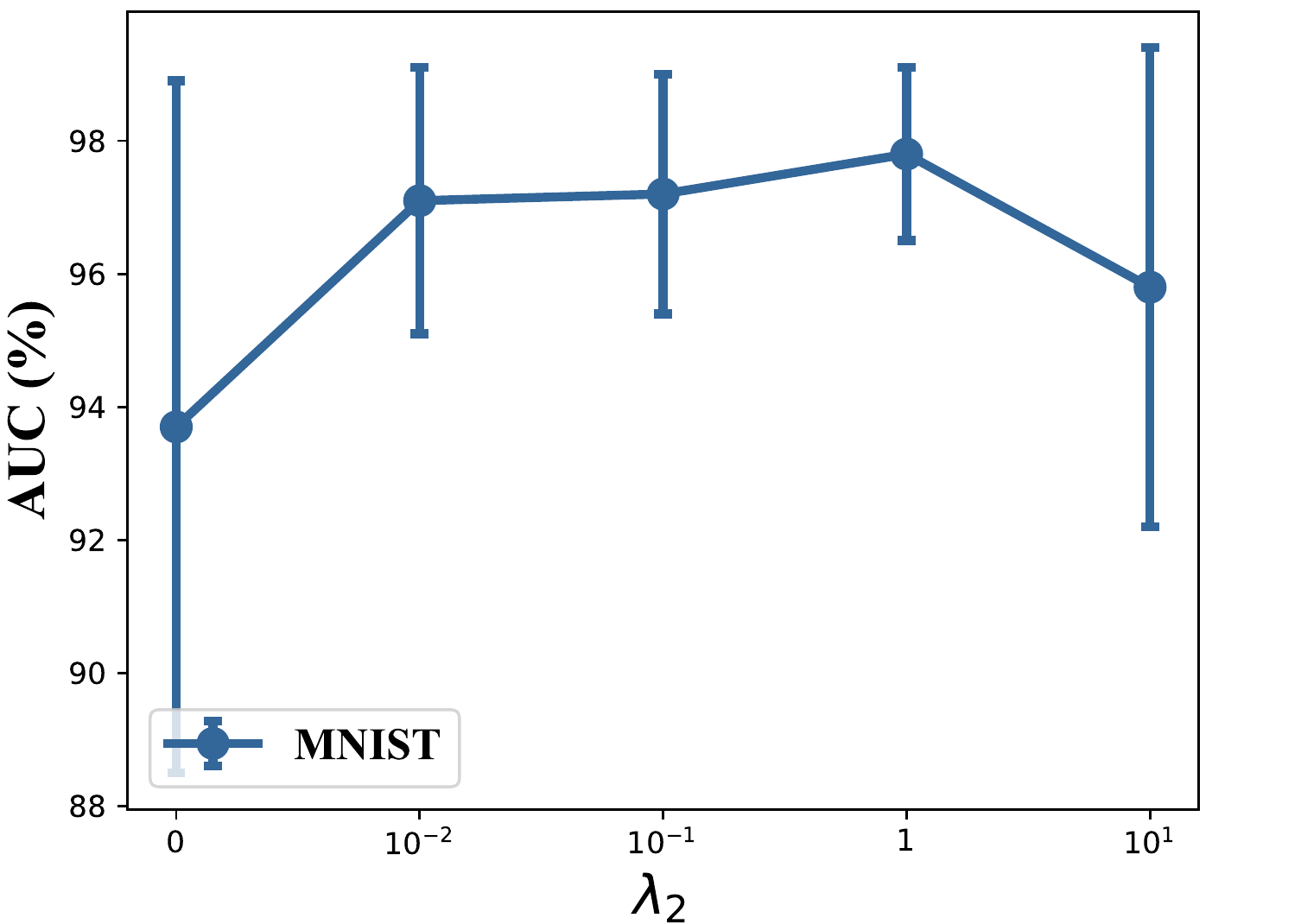}
\caption{ESAD sensitivity analysis w.r.t. $\lambda_1$ on MNIST. We report avg. AUC with st. dev. over 90 experiments for various values of hyperparameter $\lambda_2$. Best viewed in color.}
\label{img:lamda2}
\vspace{-5pt}
\end{figure}

To establish hybrid methods, we apply the OC-SVM, IF, and SSAD to the resulting bottleneck representations given by the respective converged autoencoders. To complete the full learning spectrum, we also include a fully supervised deep classifier trained on the binary cross-entropy loss.

\subsection{Supplementary Experimental Results}
Besides the experiments in the main paper, we examine three scenarios~\cite{SAD} in which we vary the following three experimental parameters:
(i) $\gamma_{l}$, the ratio of labeled samples in the training data; (ii) $\gamma_{p}$, the ratio of pollution, i.e., unknown anomalies, in the unlabeled training data, and (iii) the number of anomaly classes $k_{l}$ included in the labeled training data. 

Besides the baselines considering in the main paper, we further consider several shallow unsupervised methods and deep unsupervised anomaly detection competitors as baselines. For the shallow unsupervised methods, OC-SVM~\cite{scholkopf2001estimating} and Isolation Forest~\cite{liu2008isolation} are considered. For the deep unsupervised anomaly detection competitors, we consider CAE \cite{masci2011stacked}, Deep SVDD \cite{ruff2018deep} AnoGAN \cite{schlegl2017unsupervised}, ALOCC \cite{Sabokrou2018Adversarially}, ADGAN \cite{deecke2018anomaly}, OCGAN \cite{perera2019ocgan}, GANomaly \cite{akccay2019skip}, P-KDGAN \cite{zhang2020p} and DGEO \cite{golan2018deep}. 
OC-SVM here have unfair advantages by selecting their hyperparameters to maximize AUC on a subset $(10\%)$ of the test set to establish strong baselines.

\noindent\textbf{Experimental Scenario (i).} For the experimental scenario (i), where the effectiveness of adding labeled anomalies during training is investigated, i.e., increasing $\gamma_{l}$, has been shown in the main paper. In this part, we further report the results comparing with several unsupervised methods under the unsupervised setting in Table~\ref{tal:un}. We emphasize that our ESAD is not designed for the unsupervised setting. In these experiments, the semi-supervised terms are not working and make ESAD incomplete, since it remains only the unsupervised terms. Thus, these results are somewhat unfair for ESAD. Note that this paper still focus on the semi-supervised setting but not the unsupervised setting.

\noindent\textbf{Experimental Scenario (ii).} For the experimental scenario (ii), where the robustness is investigated in this scenario through adding polluted data. With an increasing pollution ratio $\gamma_{p}$, we pollute the unlabeled training set with anomalies drawn from all nine anomaly classes. We fix $\gamma_{l}=0.05$ in this scenario. We report the average results over 90 experiments per pollution ratio $\gamma_{p}$. The corresponding results are shown in Table~\ref{tal:2}. Results show that ESAD is least affected by the pollution data and show the best robustness in all the polluted levels.

\noindent\textbf{Experimental Scenario (iii).} For the experimental scenario (iii), we increase the number of anomaly classes $k_{l}$ included in the labeled part of the training set to increase the diversity of labeled anomalous data. As shown in Table~\ref{tal:3}, ESAD shows better performance in this scenario. For example, the AUC of ESAD on CIFAR-10 increases from 81.9\% to 86.7\% ($\gamma_{l}=0.05$, $\gamma_{p}=0.1$) when we change $k_{l}$ from 1 to 5.

\noindent \textbf{Examples Visualization.} We illustrate the predictions of our model in Figure~\ref{img:vis_medical}. Samples are randomly chosen from H\&E-stained lymph node of Camelyon16 challenge \cite{bejnordi2017diagnostic} (top) and chest X-rays of NIH dataset \cite{wang2017chestx} (bottom). These samples and their corresponding scores show that the higher the score, the more likely to be an anomaly.

\noindent \textbf{Sensitivity Analysis on $\lambda_2$.}
We analyze the sensitivity of ESAD over the hyperparameters $\lambda_2$. Figure~\ref{img:lamda2} shows the performance with different $\lambda_2$ using ESAD on MNIST. We set $\gamma_l=0.05$, $\gamma_p=0.1$, $k_l=1$ in this experiment. Results show that without the assistant loss, i.e., $\lambda_2 = 0$, ESAD shows relatively low and unstable AUCs. ESAD shows best performance when $\lambda_2 = 1$. When $\lambda_2$ is too large, ESAD also shows unstable performance. This is because both two encoders will converge into the same constant function if the impact of the assistant loss is much greater than the other two mutual information and entropy based loss functions.

\end{document}